\documentclass[11pt]{article}
\usepackage[utf8]{inputenc}
\usepackage{amsmath}
\usepackage{amsfonts}
\usepackage{amssymb}
\usepackage{amsthm}
\usepackage{subcaption}
\usepackage{graphicx}

\usepackage{mathtools}
\usepackage{enumerate}
\usepackage{enumitem}
\usepackage[margin=1in]{geometry}
\usepackage{url}
\usepackage{algorithm}
\usepackage{algorithmic}
\usepackage{commath}
\usepackage[sort]{natbib}
\usepackage{hyperref}
\usepackage{booktabs}
\usepackage{comment}
\usepackage{xcolor}

\usepackage{tikz-network}
\usepackage{tikz}
\usepackage{tikz-3dplot}

\numberwithin{equation}{section}

\usepackage{multirow}

\usepackage{mathfont}

\newtheorem{theorem}{Theorem}[section]

\newtheorem{lemma}[theorem]{Lemma}

\newtheorem{assumption}[theorem]{Assumption}
\newtheorem{definition}[theorem]{Definition}

\newcommand{\Dhat}{\hat{D}}

\newcommand{\reps}{r_{\varepsilon}}
\newcommand{\repst}{r_{\varepsilon_t}}
\newcommand{\phieps}{\phi_{\varepsilon}}
\newcommand{\phiepst}{\phi_{\varepsilon_t}}

\newcommand{\epstol}{\varepsilon_{\mathrm{tol}}}

\newcommand{\inprd}[2]{\langle #1, #2 \rangle}

\title{Transferable Optimization Network for Cross-Domain Image Reconstruction}
\author{Yunmei Chen\thanks{Department of Mathematics, University of Florida, Gainesville, FL 32611, USA. Email: \url{yun@ufl.edu}.} \and Chi Ding\thanks{Department of Mathematics, University of Florida, Gainesville, FL 32611, USA. Email: \url{chiding.uf@gmail.com}.} \and Xiaojing Ye\thanks{Department of Mathematics and Statistics, Georgia State University, Atlanta, GA 30303, USA. Email: \url{xye@gsu.edu}.}}
\date{}

\begin{document}

\maketitle

\begin{abstract}
We develop a novel transfer learning framework to tackle the challenge of limited training data in image reconstruction problems. The proposed framework consists of two training steps, both of which are formed as bi-level optimizations. In the first step, we train a powerful universal feature-extractor that is capable of learning important knowledge from large, heterogeneous data sets in various domains. In the second step, we train a task-specific domain-adapter for a new target domain or task with only a limited amount of data available for training.  Then the composition of the adapter and the universal feature-extractor effectively explores feature which serve as an important component of image regularization for the new domains, and this leads to high-quality reconstruction despite the data limitation issue. We apply this framework to reconstruct under-sampled MR images with limited data by using a collection of diverse data samples from different domains, such as images of other anatomies, measurements of various sampling ratios, and even different image modalities, including natural images. Experimental results demonstrate a promising transfer learning capability of the proposed method.

\bigskip


\bigskip


\end{abstract}

\section{Introduction}
\label{sec:intro}

Deep learning (DL) \cite{goodfellow2016deep} has demonstrated impressive performance in a variety of real-world applications during the past decade. 
It is considered as a promising approach to general artificial intelligence (AI).
However, most existing DL methods are extremely data demanding \cite{simonyan2014very, ahishakiye2021survey, fabian2022humus-net, schlemper2018deep} and require training and testing data to follow the same probability distribution to attain the desired quality of the solution. 
This issue has notoriously hindered the application of deep learning to a broad range of real-world problems where it is infeasible, difficult, or expensive to obtain a sufficient amount of training data. 
Moreover, real-world data sets are often collected from different sources and exhibit substantial heterogeneity with inconsistent probability distributions in practice \cite{bao2023recent}. 
Therefore, a core topic in DL research is to learn knowledge from source domains with sufficient amounts of data, and leverage such knowledge to new target domains with limited data. 
Transfer learning (TL) is one of the most promising approaches to address this problem \cite{pan2010survey,neyshabur2020what,tan2018survey}.

A variety of TL approaches have been proposed in the computer science community in the past decades \cite{pan2010survey,tan2018survey}. 
However, the performance of existing approaches often degrades when the data distributions in the source and target domains differ significantly, or the data in target domain are very limited, especially for large-scale problems. 
Therefore, recent advances of TL emphasize more on efficiency, scalability, and cross-domain adaptability. 
%
%
Hence, the success of a TL method relies on its capability of handling related issues, including limited data, interpretability, privacy concerns, and computational efficiency for real-world applications.

This work aims at developing a novel TL approach to address the challenge mentioned above. Our approach intends to combine the merits of variational modeling, bi-level optimization, and the deep unrolling network technique. 
First, variational model is a classic mathematical formulation to solve inverse problems, such as image reconstruction, by combining both data consistency and prior knowledge information as regularization to eliminate the ill-posedness issue of inverse problems and improve solution quality. 
Second, bi-level optimization has received significant interests in learning regularization for inverse problems in the past few years. 
Lastly, unrolling networks have proven to be more effective than manually designed networks as they integrate consistency with data into their network architectures. 
While unrolling networks have state-of-the-art performance, their adoption in solving TL problems is still not explored in the literature. 
A main feature of our proposed approach is to establish a unified framework of employing unrolling networks with bi-level optimization and variational model to solve TL problems.

Our approach constructs two types of nonlinear functions through specially designed bi-level optimization schemes: the first one is a powerful feature extraction mapping that learns the ability to extract the most generalizable knowledge from diverse data sets with large amounts of samples, and the second one is a set of small-scale adaptation mappings that, when composed with the feature extraction mapping, can form effective mappings that adapt the knowledge to new tasks where only limited data are available. 
We call these two maps the \emph{feature-extractor} and \emph{adapters}, respectively. 
The main contributions of this work are summarized below.
\begin{enumerate}
    \item Our approach advocates for training a powerful feature-extractor that can be adopted to many different domains with limited data. The feature-extractor is trained via a properly designed bi-level optimization scheme to handle diverse and heterogeneous data sets.
    
    \item To tackle the issue of limited data in new domains or tasks, we train small-scale and tasks-specific adapters. The composition of such an adapter and the trained feature-extractor is capable of achieving high solution quality. The training of these adapters is also performed by solving specified bi-level optimization schemes.

    \item We provide several techniques to improve the empirical performance of the proposed TL method. These include a data augmentation method for improved information exploration in limited-data scenarios and initialization methods to mitigate the non-convexity issue in network training.
    
    \item We conduct a variety of numerical experiments on different TL problems. In this work, we focus on TL in the application of magnetic resonance (MR) image reconstruction, where most of the aforementioned challenges in TL will be considered and tested. 

\end{enumerate}

\paragraph{Paper organization}
In the remainder of this paper, we first provide a brief review of TL methods for image reconstruction developed in the past few decades in Section \ref{sec:related_work}. 
Our proposed TL method that uses variational formulation integrated in bi-level optimization schemes is described in details in Section \ref{sec:method}.
A numerical algorithm with rigorous convergence analysis, including computation complexity, for solving the lower-level optimization problems is provided in Section \ref{sec:elda}.
Numerical experiments and results to evaluate the proposed method are presented in Section \ref{sec:experiment}.
Section \ref{sec:conclusion} concludes the paper.

\paragraph{Notations}
We use $\langle \cdot, \cdot \rangle$ and $\|\cdot \|$ to denote the standard inner product and 2-norm, respectively, in real and complex Euclidean spaces. For any matrix $\zbf \in \mathbb{C}^{d_1 \times d_2}$, its (2,1)-norm is denoted by $\|\zbf\|_{2,1}:=\sum_{k=1}^{d_1} \|\zbf_k\|$, where $\zbf_k \in \mathbb{C}^{d_2}$ is the $i$th row of $\zbf$. We let $[N]:=\{1,\dots,N\}$ for any $N \in \mathbb{N}$. We let $\lfloor a \rfloor$ denote the floor of $a \in \Rbb$, i.e., the maximum integer that is smaller than or equal to $a$. If $S$ is a subset in $\mathbb{C}^{n}$, then $c+S := \{c + s\ |\ s \in S\}$ for any $c \in \mathbb{C}^{n}$. For a function $\gbf$ parameterized as a deep neural network, by solving $\gbf$ from a minimization problem we mean to find the optimal network parameters of $\gbf$, and therefore will not introduce an additional notation to denote its parameters.

\section{Related work}
\label{sec:related_work}

TL is closely related to but different from multitask learning (MTL) and meta-learning (ML) \cite{vettoruzzo2024advances}. 
The goal of TL is to leverage representations learned from one or more source tasks to solve a target task.
 On the contrary, MTL aims at improving performance on a set of tasks (or domains) by learning them simultaneously but does not consider transferring knowledge to new tasks, while the goal of ML is to learn the ability to transfer knowledge from past tasks to solve new tasks \cite{vettoruzzo2024advances}. 
A common critical challenge to these three fields is data scarcity particularly in target domains or tasks, and mismatched training and test data distributions.
Since our approach falls into the class of TL, we will focus on the literature of TL in the remainder of this section.

{TL has effectively addressed the aforementioned challenges by its ability of leveraging knowledge gained from one task or domain and applying to a different but related one, even if only limited labeled data are available or two domains differ in feature space or data distribution \cite{pan2010survey} in variety of applications; a number of surveys have been conducted on TL to outline and analyze various TL techniques from different aspects, e.g. \cite{tan2018survey, Niu2020_Decade, Zhuang2021_Comprehensive, masoume2025review}}.

Since the emergence of DL, recent advances of TL usually perform network training in source domains where large data sets are available, and knowledge captured by the trained network is transferred to different target domains \cite{donahue2014decaf, yosinski2014transferable}. The authors of \cite{yosinski2014transferable} point out the relationship between network structure and transferability, and demonstrate that some network modules may influence the transferability but not in-domain accuracy. Moreover, they provide descriptions of transferrable features and network types that are more suitable for transfer.

The majority of deep TL are implemented through the means of fine-tuning \cite{vettoruzzo2024advances}. For example, certain network architectures (e.g., VGG \cite{simonyan2014very} and ResNet \cite{he2016deep}) are shown to be effective when pre-trained on large data sets (such as ImageNet \cite{deng2009imagenet}) and fine-tuned for new tasks. 
In \cite{oquab2014learning}, it is shown that the front layers of a convolutional neural network (CNN) pre-trained on ImageNet can be used as intermediate representations, and these representations can be transferred to other visual recognition tasks.
In \cite{chang2017unsupervised}, a multi-scale convolutional sparse coding method is proposed to learn the convolution filter bank and the representation coefficients jointly on the source domains and fine-tune the based knowledge to new domains. 
The work \cite{ABUDALFA2025transfer} uses pre-trained computer-vision (CV) models (ResNet, VGG, and AlexNet)  for a classification of fall and non-fall events. Then the CV model is refined by using simulated data representing various fall scenarios to perform fall detection from UWB radar data.
In \cite{park2025efficient}, pre-trained unrolling networks are fine-tuned by supervised learning, schooling, and unsupervised learning; and empirical results are shown to demonstrate the performance of these three different methods.

{ Recently, parameter-efficient fine-tuning (PEFT) methods, such as Low-Rank Adaptation (LoRA) \cite{zeng2024expressive}, and adapters have been applied to make transfer learning more cost efficient}.
In \cite{rebuffi2017learning}, the authors introduce residual adapter modules into a ResNet, allowing the model to share most parameters across domains while training only a small number of domain-specific components.This approach avoids the forgetting problem and obtains comparable to full fine-tuning.  {In \cite{shen2021partial, lee2022surgical, chen2024conv}, it is shown that transferring partial knowledge by selectively freezing or fine-tuning a subset of layers matches or outperforms commonly used fine-tuning approaches.} 
%
%
%
In \cite{shen2024med-tuning} a med-tuning framework realizes PEFT for medical volumetric segmentation, in which an efficient plug-and-play module named Med-Adapter is applied for task-specific feature extraction. 
{The work \cite{bafghi2025fine}
proposes an extension of PEFT method, using an indicator function to selectively activate LoRA blocks. The approach minimizes knowledge loss, retains its generalization strengths under domain shifts, and significantly reduces computational costs compared to traditional fine-tuning. }
Different from typical fine-tuning, the authors of \cite{dar2020transfer-learning} propose a network-based deep TL approach. A deep reconstruction network is trained using a loss function consisting of least squares and absolute errors in data fitting and the Tikhonov regularization of network parameters, and the trained network is used to obtain crude reference for regularization with data fitting to reconstruct the image associated to new measurement data.

{Besides parameter-efficient adaptation, recent TL methods focus more on multi-modal pre-training and cross-domain knowledge transfer for specialized tasks. Cross-domain TL methods apply to scenarios where the domains have different feature spaces, and this requires bridging different feature spaces or data distributions \cite{bao2023recent}. 
One of the approaches to handle heterogeneity of feature spaces and more related to the present work is incorporating domain-invariant features and domain-specific features from each domain. 
Extracting meaningful and universal feature from diverse data or task often leads to better feature representations and the quality of TL. Moreover, knowledge distillation techniques are used more often to help this process and reduce computational cost, especially for recent large-scale models, such as Vision-Language Models (VLMs) \cite{mansourian2025a}.
In \cite{girdhar2023one} the ``IMAGEBIND'' approach binds of  different modalities (images, text, audio, depth, thermal, and IMU data) in a joint embedding space by using only images to bind them together enabling  transfer to unseen modality combinations.}
The work \cite{liu2021universal} trains a universal network to effectively reconstruct high-quality MRI images from under-sampled data across various anatomies (e.g., brain, knee, abdomen).
To compensate for statistical shift and capture the knowledge specific to each data set, an anatomy-specific instance normalization (ASPIN) module is implemented to accompany the universal model for the reconstruction of multiple anatomies. {The base network captures shared knowledge, andASPIN modules are added for anatomy-specific adjustments. The universal models are trained by distilling knowledge \cite{hinton2015distilling} from pre-trained, anatomy-specific models. } For new anatomy data set the ASPIN parameters are trained with the base universal model parameters fixed. 
The work \cite{bian2021optimization-based} learns a universal sparse representation with task-specific weighting from diverse data sets as the regularization in a variational model. The deep neural network unfolds a learned descent-type algorithm to reconstruct multiple MR images of the same objects with the same contrast but different sampling ratios and masks.
{In a recent work \cite{gao2025knowledge}, a novel knowledge transfer methodology termed as Learning from Interactions (LFI) that
explicitly models visual understanding as an interactive process. The key idea is that capturing the dynamic interaction patterns encoded in pre-trained Vision-Language Models (VLMs) for efficient knowledge transfer to visual foundation models (VFMs). The interaction between VLMs and VFMs means understanding the world through analogous cognitive processes, despite of the final form of representation. By refining these representation-agnostic interactions, LFI facilitates cross-modal and cross-task knowledge transfer.}

\paragraph{Our contributions}
This work presents a novel TL approach that is vastly different from all existing ones in the following aspects. 
Existing TL approaches rely on various heuristics for network architecture design to transfer knowledge learned from source domains with large data sets to target domains where data is usually limited. 
In contrast, we propose the first TL approach that integrates classic variational modeling and nonsmooth nonconvex bi-level optimization in a deep learning framework for improved interpretability and rigorous convergence. 
Our approach adopts the idea of feature-extractor and adapter network modules in a learnable regularization for TL in image reconstruction, and demonstrates outstanding performance in knowledge transferability and parameter efficiency when compared to the state-of-the-art methods.

\section{Method}
\label{sec:method}

Our proposed TL method consists of two major training steps. In Step 1, we train the feature-extractor from large and heterogeneous data set. These data sets often contain multiple and diverse subsets of many samples, but they are collected from different task domains. 
For instance, each subset may contain a sufficient amount of MR images of a specific anatomical region, such as brain or knee. Along with the feature-extractor, one adapter for each of these subsets is also trained to ensure the feature-extractor indeed learns the universal information from the entire data set. This step is formulated as a bi-level optimization problem with the feature-extractor and all adapters as unknowns to be solved. 
The solution to such a lower-level optimization is a reconstructed image, whose deviation from its corresponding ground truth image is the loss function of the upper-level optimization.

In Step 2, we employ the feature-extractor trained in Step 1 and learn several small-scale adapters, each on a small data set collected for a new task. These data sets may be, for example, cardiac or prostate MR images. Step 2 can also be formulated as a bi-level optimization problem, but only these new adapters are trained.

In what follows, we assume $D := \{D_i: i \in [I]\}$ contains $I$ data subsets, and each $D_i$ has sufficient training samples for a specific domain or task. We train the feature-extractor, denoted by $\gbf$, using $D$ in Step 1. The adapter for each $D_i$ is denoted by $\hbf_i$, which is trained together with $\gbf$. 
We have another $J$ data sets $\Dhat := \{\Dhat_j: j \in [J]\}$, each $\Dhat_j$ only has a limited amount of samples, and they are used for training their corresponding adapters, denoted by $\hat{\hbf}_j$, in Step 2. Suppose $|D_i|=M$ and $|\Dhat_j|=N$, and there is often $N \ll M$.
For instance, we consider for simplicity $I=J=2$. Let $D=\{D_1,D_2\}$, where $D_1$ and $D_2$ are two data sets consisting of $M$ brain and knee MR images, respectively. 
Then we have $D_1=\{(\ybf_m^{(1)},\,\hat{\xbf}_{m}^{(1)}): m \in [M]\}$, where $\ybf_{m}^{(1)}$ is an under-sampled brain MR image measurement data in the Fourier space, also known as the $k$-space, and $\hat{\xbf}_{m}^{(1)}$ is the corresponding ground truth image. 
Here $\hat{\xbf}_{m}^{1}$ is a 2D gray-scale MR brain image of size (resolution) $\sqrt{n} \times \sqrt{n}$, which can be rearranged as a column vector in $\mathbb{C}^{n}$ by stacking its columns. 
The physical relation between $\ybf_{m}^{(1)}$ and $\hat{\xbf}_{m}^{(1)}$ is formulated as $\ybf_{m}^{(1)} = \Pbf \Fbf \hat{\xbf}_{m}^{(1)} + \epsilon$, where $\Fbf$ stands for the discrete Fourier transform matrix, $\Pbf$ is a binary matrix describing the sampling mask, and $\epsilon$ is some unknown Gaussian noise with zero mean. Then the data fidelity is given by
\begin{equation}
\label{eq:data-fidelity}
f(\hat{\xbf}_{m}^{(1)}; \ybf_{m}^{(1)}) = \frac{1}{2} \| \Pbf \Fbf \hat{\xbf}_{m}^{(1)} - \ybf_{m}^{(1)} \|^2,
\end{equation}
which is identical to the negative log-likelihood of $\ybf_{m}^{(1)}$ given $\hat{\xbf}_{m}^{(1)}$ up to a constant.
We assumed $\epsilon$ to be Gaussian noise which yields $f$ in \eqref{eq:data-fidelity}, but it can be easily adjusted according to other types of noise.
In this work, we just assume $f$ to be a differentiable and possibly non-convex function with an $L_f$-Lipschitz continuous gradient $\nabla f$ for some $L_f >0$. We remark that $\nabla f$ can be just locally Lipschitz and it does not affect the convergence properties of the algorithm we will develop later.

\paragraph{Step 1: Train the feature-extractor $\gbf$}
We parameterize $\gbf$ as a deep CNN whose structure in our implementation will be given in Section \ref{sec:experiment}. Along with $\gbf$, we will also obtain an adapter $\hbf_i$ for each data set $D_i$. These adapters are parameterized as small-size CNNs whose architectures will also be given later. 
We obtain $\gbf$ by solving the following bi-level optimization problem:
\begin{subequations}
\label{eq:solve-g}
\begin{align}
    \min_{\gbf, \{\hbf_i\}} \quad & \frac{1}{IM} \sum_{i=1}^{I} \sum_{m = 1}^{M} \Big\{ \|\xbf_m^{(i)}-\hat{\xbf}_m^{(i)}\|^2  - \alpha\, S(\xbf_m^{(i)},\hat{\xbf}_m^{(i)}) \Big\} \label{eq:solve-g-upper}
   \\
    \text{s.t.} \quad & \xbf_m^{(i)} = \argmin_{\xbf}\ \{ f(\xbf; \ybf_m^{(i)}) + \|\hbf_i(\gbf(\xbf))\|_{2,1} \},\quad \forall\, m \in [M], \ i \in [I], \label{eq:solve-g-lower}
\end{align}
\end{subequations}
where $\alpha > 0$ is a manually chosen weight parameter.
In \eqref{eq:solve-g-upper}, $S$ stands for some image similarity measure. In this work, we set $S$ to the Structural Similarity Index Measure (SSIM) \cite{wang2004image}, which is a means to evaluate the similarity between two images. SSIM has a range $[0,1]$, and the more two images are similar in local and global structures, the larger their SSIM value is. 
Notice that an adapter $\hbf_i$ is also learned for $D_i$.
Their role is to justify that $\gbf$ is a capable feature-extractor to extract basic features from the diverse data sets in $D$.

\paragraph{Step 2: Train the adapters $\{\hat{\hbf}_j: j \in [J]\}$}

We use a similar bi-level optimization model to learn the adapters $\hat{\hbf}_j$'s for the small data sets in $\Dhat$ as follows:
\begin{subequations}
\label{eq:solve-h}    
\begin{align}
    \min_{\hat{\hbf}_j} \quad & \frac{1}{N} \sum_{n = 1}^{N} \Big\{ \|\xbf_n^{(j)}-\hat{\xbf}_n^{(j)}\|^2
    - \alpha\, S(\xbf_n^{(j)},\hat{\xbf}_n^{(j)}) \Big\} \label{eq:solve-h-upper}
   \\
    \text{s.t.} \quad & \xbf_n^{(j)} = \argmin_{\xbf}\ \{ f(\xbf; \ybf_n^{(j)}) + \|\hat{\hbf}_j( \gbf (\xbf))\|_{2,1} \},\quad \forall\, n \in [N], \ j \in [J]. \label{eq:solve-h-lower}
\end{align}
\end{subequations}	
Notice that only these adapters $\hat{\hbf}_j$'s are learned while the feature-extractor $\gbf$ is fixed in \eqref{eq:solve-h}. These adapters have the same small architecture as the previous adapters $\hbf_i$'s, and the compositions of these adapters and $\gbf$ become effective feature mappings for the small data sets in $\Dhat$. 
The reason of using small networks for these adapters is that they can still be properly trained even with small data.

\paragraph{Testing}
Once we complete Steps 1 and 2 and obtain the feature-extractor $\gbf$ and all the adapters $\hat{\hbf}_j$'s, we can apply $\hat{\hbf}_j \circ \gbf$ in \eqref{eq:solve-h-lower} and reconstruct the image $\xbf$ from any new measurement $\ybf$ following the same distribution as those in $\Dhat_j$. 
The quality of this $\xbf$ will be much better than those produced by methods which only use the small data $\Dhat_j$ for training, due to the ability of $\gbf$ in learning important features trained by large and heterogeneous data set $D$ and the adaptation by $\hat{\hbf}_j$ in the new task on $\Dhat_j$.

As we can see, the key to accomplishing Steps 1 and 2 is an effective scheme for solving the bi-level optimization problems \eqref{eq:solve-g} and \eqref{eq:solve-h}, whose lower-level problems \eqref{eq:solve-g-lower} and \eqref{eq:solve-h-lower} are non-smooth and non-convex. In this work, we modify the efficient learnable descent algorithm (ELDA) \cite{zhang2024provably} to improve its computation complexity and apply it to solve these problems. Similar as ELDA, our method naturally induces unrolling networks such that $\gbf$, $\hbf_i$'s and $\hat{\hbf}_j$'s can be trained by the given data and the problems \eqref{eq:solve-g} and \eqref{eq:solve-h} are solved.

\section{Algorithm}
\label{sec:elda}

We introduce a modified version of ELDA and provide a comprehensive convergence analysis including its computation complexity to achieve an $\epstol$-optimal solution for any prescribed tolerance $\epstol>0$. 
This new ELDA version can be applied to solve general non-convex non-smooth optimization problems, but we restrict its usage in solving the lower-level optimization problems \eqref{eq:solve-g-lower} and \eqref{eq:solve-h-lower} with known $\{\gbf,\hbf_i,\hat{\hbf}_j\}$ in this work.
Compared to the original version of ELDA \cite{zhang2024provably}, the new version has an improved computation complexity by using a more properly designed descent condition \eqref{eq:v-condition} below. 
This condition yields a different range of step size for line search and reduce the theoretical computation complexity from $O(\epstol^{-4})$ to $O(\epstol^{-3})$, as shown in Theorem \ref{convergence-thm}.

In this work, we set $\gbf: \mathbb{C}^{n} \to \mathbb{C}^{m_1 \times m_2}$ to be a convolutional neural network (CNN) such that $\xbf$ is mapped to a 2D matrix (tensor) $\gbf(\xbf)$. Here $m_1$ is the number of pixels in $\gbf(\xbf)$, which depends on the choice of stride and padding of the CNN, and $m_2$ is the dimension of the feature vector at each pixel. 
For example, if $\gbf$ is set to be a CNN with stride 1, zero-padding and feature dimension 16, then $m_1=n$ and $m_2=16$. Similarly, all the adapters are also parameterized as a CNN with the same architecture but different parameters, and they yields mappings from $\mathbb{C}^{m_1\times m_2}$ to $\mathbb{C}^{d_1\times d_2}$. 
We use smooth functions with globally Lipschitz continuous gradients as the activation functions of these CNNs throughout this work.
Therefore the composition $\qbf := \hbf \circ \gbf: \mathbb{C}^{n} \to \mathbb{C}^{d_1 \times d_2}$ is smooth and $\nabla \qbf_k$ is Lipschitz continuous, where $\hbf$ can be any of the adapters, i.e., $\hbf_i$ or $\hat{\hbf}_j$, and $\qbf_k(\xbf) \in \mathbb{C}^{d_2}$ is the $k$th row of $\qbf(\xbf) \in \mathbb{C}^{d_1 \times d_2}$ for any $\xbf \in \mathbb{C}^{n}$.
Since the measurement data $\ybf$ is given in these problems, we drop $\ybf$ in this section for notation simplicity. 

Now both \eqref{eq:solve-g-lower} and \eqref{eq:solve-h-lower} can be written in a general form as:
\begin{equation}
\label{eq:formulation}
    \min_{\xbf}\ \{\ \phi(\xbf):=f(\xbf; \ybf)+r(\xbf) \ \},
\end{equation}
where the regularization is defined by
\begin{equation}
    \label{eq:r}
    r(\xbf) = \|\qbf(\xbf)\|_{2,1} = \sum_{k=1}^{d_1} \|\qbf_k(\xbf) \|.
\end{equation}
Notice that $r$ is non-smooth due to the $2$-norm and non-convex due to the flexible network architectures of $\gbf$ and $\hbf$.

\subsection{Modified efficient learnable descent algorithm}
\label{subsec:elda}

To tackle the non-smooth and non-convex $r$, we follow the original ELDA to mollify $r$ such that it becomes smooth and non-convex: for any $\varepsilon>0$ and $\xbf$, define
\begin{equation}
    \label{eq:smoothedR}
        r_\varepsilon(\xbf) := \sum_{k\in K_0^{\varepsilon}(\xbf)}\frac{1}{2\varepsilon}\norm{\qbf_k(\xbf)}^2 + \sum_{k \in K_1^{\varepsilon}(\xbf)}\Big( \norm{\qbf_k(\xbf)}-\frac \varepsilon 2\Big)
\end{equation}
where 
\begin{equation}
\label{eq:def-K01}
    K_0^{\varepsilon}(\xbf):=\{k\in \left[d_1\right]:\norm{\qbf_k(\xbf)}\leq \varepsilon\},\quad K_1^{\varepsilon}(\xbf): = \left[d_1\right] \setminus K_0^{\varepsilon}(\xbf).
\end{equation}
It is easy to see that
\begin{equation}\label{eq:r_eps_bound}
\reps(\xbf)\leq r(\xbf) \leq \reps(\xbf) +\frac{d_1\varepsilon}{2},\quad \forall\, \xbf \in \mathbb{C}^n.
\end{equation}
The objective function with this modified regularization $\reps$ is denoted by
\begin{equation}
    \label{eq:phi-eps}
    \phi_{\varepsilon}(\xbf) := f(\xbf) + \reps(\xbf).
\end{equation}
This mollification changes the original objective function $\phi$ to $\phi_{\varepsilon}$. However, we can address this issue by letting our algorithm automatically decrease $\varepsilon$ during iterations and prove its desired convergence properties in solving \eqref{eq:formulation} with the original $\phi$.

Now we employ the following modified proximal gradient descent algorithm
to solve the smoothed minimization problem \eqref{eq:phi-eps}:
\begin{subequations}
\label{eq:elda-subproblems}
\begin{align}
    \zbf_{t+1} & = \xbf_t - \alpha_t \nabla f(\xbf_t), \\
    \ubf_{t+1} & = \argmin_{\ubf} \Big\{ \repst(\zbf_{t+1}) + \langle \nabla \repst(\zbf_{t+1}),\ubf-\zbf_{t+1} \rangle + \Big(\frac{1}{2\gamma_t} + \frac{1}{2\alpha_t} \Big) \| \ubf - \zbf_{t+1}\|^2\Big\},\label{eq:elda-u-subproblem} 
\end{align}
\end{subequations}
where $t$ is the iteration number, $\alpha_t,\gamma_t>0$ are the step sizes, and $\varepsilon_t>0$ is the smoothing level in iteration $t$. The modification to the proximal gradient algorithm is  to replace the $\repst(\xbf)$ by its linear approximation at $\zbf_{t+1}$ plus a quadratic penalty $\frac{1}{2\gamma_t}\| \ubf - \zbf_{t+1}\|^2$ to match the architecture of residual leaning for better training.
We can easily see that \eqref{eq:elda-subproblems} reduces to the following updates with $\beta_t:=\frac{\alpha_t \gamma_t}{\alpha_t + \gamma_t}$:
\begin{subequations}
\label{eq:elda-update}
\begin{align}
    \zbf_{t+1} & = \xbf_t - \alpha_t \nabla f(\xbf_t), \label{eq:elda-z} \\
    \ubf_{t+1} & = \zbf_{t+1} - \beta_t \nabla \repst(\zbf_{t+1}). \label{eq:elda-u} 
\end{align}

\end{subequations}

Let $\eta_1,\eta_2,\eta_3>0$ be any three prescribed parameters for our algorithm, and we first check the following two conditions for $\ubf_{t+1}$:
\begin{subequations}
\label{eq:elda-condition}
\begin{align}
    \| \nabla \phiepst (\xbf_t) \| \  & \le \ \frac{1}{\eta_1} \| \ubf_{t+1} - \xbf_t \|, \label{eq:elda-condition1}\\
    \phiepst(\ubf_{t+1}) - \phiepst(\xbf_t) \  & \le - \frac{\eta_2}{2} \| \ubf_{t+1} - \xbf_t \|^2. \label{eq:elda-condition2}
\end{align}
\end{subequations}

If $\ubf_{t+1}$ satisfies the two conditions in \eqref{eq:elda-condition}, then we set $\xbf_{t+1}=\ubf_{t+1}$ and directly continue to the next iteration and compute \eqref{eq:elda-update}.
Otherwise, we compute another variable $\vbf$ as follows,
\begin{align}
    \label{eq:v}
    \vbf_{t+1}& =\argmin_v \Big\{ \phiepst(\xbf_t) + \inprd{\nabla \phiepst(\xbf_t)}{v-\xbf_t}+\frac{1}{2\bar{\alpha}_t}\norm{v-\xbf_t}^2 \Big\} \nonumber \\
    & = \xbf_t - \bar{\alpha}_t \nabla \phiepst(\xbf_t)
\end{align}
where $\bar{\alpha}_t = \bar{\alpha}$, and $\bar{\alpha}>0$ is a pre-selected and fixed parameter.
Then we check whether the following condition is satisfied:
\begin{equation}
    \label{eq:v-condition}
    \phiepst(\vbf_{t+1}) - \phiepst(\xbf_t) \le - \frac{\eta_3}{\varepsilon_{t}} \| \vbf_{t+1} - \xbf_t \|^2.
\end{equation}
If \eqref{eq:v-condition} is met, then we set $\xbf_{t+1}=\vbf_{t+1}$ and move on to the next iteration to compute \eqref{eq:elda-update}; Otherwise we will adopt the standard line search technique by reducing the step size $\bar{\alpha}_t$ in \eqref{eq:v} to $\rho \bar{\alpha}_t$ with some pre-selected parameter $\rho \in (0,1)$ until \eqref{eq:v-condition} is satisfied. 
It is known that such line search will always stop in finitely many steps to make \eqref{eq:v-condition} hold because $\nabla \phieps$ is Lipschitz continuous (with Lipschitz constant depending on $\varepsilon_t$ as shown below) at any iteration $t$. 
The strategy to automatically decrease $\varepsilon_t$ is also integrated when $\|\nabla \phiepst(\xbf_{t+1})\|$ is sufficiently small. The modified ELDA for solving \eqref{eq:formulation} is given in Algorithm \ref{alg:elda}. 
A flowchart of the $t$th iteration of this algorithm, which correspond to the $t$th phase of the proposed network, is plotted in Figure \ref{fig:flowchart}.
The details about the convergence properties of Algorithm \ref{alg:elda} will be given in Section \ref{subsec:elda-convergence}.

\begin{algorithm}[htb]
\caption{Modified ELDA for solving \eqref{eq:formulation}}
\textbf{Input:} Parameters $\bar{\alpha},\eta_1,\eta_2,\eta_3,\gamma,\varepsilon_0,\varepsilon_{\text{tol}}>0$, $\rho,\sigma \in (0,1)$ and $T \in \mathbb{N} \cup \{\infty\}$. Initial guess $\xbf_0$.
\textbf{Output:} $\xbf_{t+1}$
\begin{algorithmic}[1]
    \FOR{$t=0,\,1,\,2,\,\dots,T$}
        \STATE $\zbf_{t+1}=\xbf_t-\alpha_t\nabla f(\xbf_t)$
        \STATE $\ubf_{t+1}=\zbf_{t+1} - \beta_t\nabla \repst(\zbf_{t+1})$
        \IF{\eqref{eq:elda-condition} is satisfied}
            \STATE $\xbf_{t+1}=\ubf_{t+1}$
        \ELSE
            \STATE $\bar\alpha_t = \bar\alpha$
            \STATE $\vbf_{t+1}=\xbf_t - \bar{\alpha}_t\nabla\phiepst(\xbf_t)$ \label{alg:v}
            \IF{\eqref{eq:v-condition} is satisfied}
            \STATE{$\xbf_{t+1}=\vbf_{t+1}$}
            \ELSE
            \STATE $\bar{\alpha}_t \leftarrow \rho \bar{\alpha}_t$ and {go to Line} \ref{alg:v}
            \ENDIF
        \ENDIF
        \IF{$\| \nabla \phiepst(\xbf_{t+1}) \| < \sigma\gamma\varepsilon_t$}
            \STATE $\varepsilon_{t+1}=\gamma\varepsilon_t$ and otherwise $\varepsilon_{t+1}=\varepsilon_t$
        \ENDIF
        \IF{$\sigma \varepsilon_t < \varepsilon_{\text{tol}}$} 
            \STATE Terminate
        \ENDIF
    \ENDFOR
    \RETURN $\xbf_{t+1}$
\end{algorithmic}
\label{alg:elda}
\end{algorithm}

\begin{figure}[h]
\centering
\begin{tikzpicture}[scale=1, transform shape, >=Latex]
\Vertex[x=0, y=0, label=$\mathbf{x}_{t}$, color=none, size=0.6]{xt} 
\node at (2,0) (umodule) [draw, thick] {$\mathbf{u}$-module};
\Vertex[x=4, y=0, label=$\mathbf{u}_{t+1}$, color=none, size=0.8]{ut1}
\node at (6.5,0) (ucond) [draw, thick] {\eqref{eq:elda-condition} holds?};
\node at (8,1.5) (vmodule) [draw, thick] {$\mathbf{v}$-module};
\Vertex[x=9.5, y=0, label=$\mathbf{x}_{t+1}$, color=none, size=0.8]{xt1} 
\node at (9.5,-1.5) (epsilon) [draw, align=left, thick] {If $\| \nabla \varphi_{\varepsilon_{t}}(\mathbf{x}_{t+1})\| < \sigma \gamma \varepsilon_{t}$, \\ then $\varepsilon_{t+1} = \gamma \varepsilon_{t}$, otherwise $\varepsilon_{t+1}=\varepsilon_{t}$ };

\Edge[Direct, lw=1pt](xt)(umodule)
\Edge[Direct, lw=1pt](umodule)(ut1)
\Edge[Direct, lw=1pt](ut1)(ucond)
\Edge[Direct, lw=1pt](vmodule)(xt1)
\Edge[Direct, lw=1pt](xt1)(epsilon)

\draw[->,-{Latex[length=2mm]},thick] (ucond) -- node[anchor=east] {no} (vmodule);
\draw[->,-{Latex[length=2mm]},thick] (ucond) -- node[anchor=north] {yes} (xt1);

\Vertex[x=0, y=-3, label=$\mathbf{x}_{t}$, color=none, size=0.6]{xt} 
\node at (2.5,-3) (df) [draw, thick] {$\mathrm{Id}-\alpha_{t}\nabla f$};
\Vertex[x=5, y=-3, label=$\mathbf{z}_{t+1}$, color=none, size=0.8]{zt1}
\node at (7.5,-3) (dr) [draw, thick] {$\mathrm{Id}-\beta_{t}\nabla r_{\varepsilon_{t}}$};
\Vertex[x=10, y=-3, label=$\mathbf{u}_{t+1}$, color=none, size=0.8]{ut1} 

\Edge[Direct, lw=1pt](xt)(df)
\Edge[Direct, lw=1pt](df)(zt1)
\Edge[Direct, lw=1pt](zt1)(dr)
\Edge[Direct, lw=1pt](dr)(ut1)

\node at (-1.5,0) {Flowchart:};
\node at (-1.5,-3) {$\mathbf{u}$-module:};
\node at (-1.5,-5) {$\mathbf{v}$-module:};

\Vertex[x=0, y=-5, label=$\mathbf{x}_{t}$, color=none, size=0.6]{xtv} 
\node at (2.5,-5) (dphi) [draw, thick] {$\mathrm{Id}-\bar{\alpha}_{t}\nabla \phi_{\varepsilon_{t}}$};
\Vertex[x=5, y=-5, label=$\mathbf{v}_{t+1}$, color=none, size=0.8]{vt1tmp} 
\node at (7.5,-5) (vcond) [draw, thick] {\eqref{eq:v-condition} holds?};
\Vertex[x=10.5, y=-5, label=$\mathbf{x}_{t+1}$, color=none, size=0.8]{vt1}

\Edge[Direct, lw=1pt](xtv)(dphi)
\Edge[Direct, lw=1pt](dphi)(vt1tmp)
\Edge[Direct, lw=1pt](vt1tmp)(vcond)
\Edge[Direct, lw=1pt](vcond)(vt1)
\Edge[Direct, lw=1pt, position=below, bend=25](vcond)(xtv)

\draw[->,-{Latex[length=2mm]},thick] (vcond) -- node[anchor=north] {yes} (vt1);
\draw (4,-6) node[below]{no, then set $\bar{\alpha}_{t} \leftarrow \rho \bar{\alpha}_{t}$};

\end{tikzpicture}
\caption{The flowchart (top) of the $t$th iteration of Algorithm \ref{alg:elda}, the $\mathbf{u}$-module (middle), and the $\mathbf{v}$-module (bottom). Each round node represents a variable. Each rectangular node represents a module, mapping, or condition check. The identity mapping is denoted by Id.}
\label{fig:flowchart}
\end{figure}
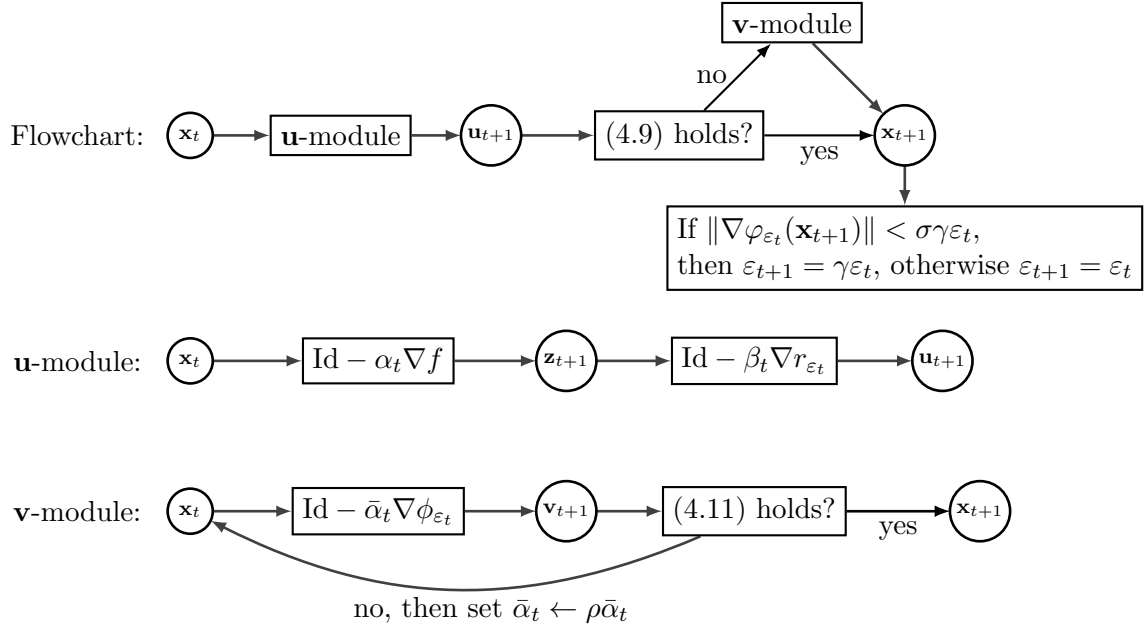

\subsection{Convergence analysis} 
\label{subsec:elda-convergence}

In this section, we provide several convergence properties, including the iteration complexity, of Algorithm \ref{alg:elda} for solving \eqref{eq:formulation}.
Since the objective function $\phi$ in \eqref{eq:formulation} is non-smooth and non-convex, we adopt the notion of \emph{Clarke subdifferential} to characterize the optimality of solutions.
\begin{definition}[Clarke subdifferential]
\label{def:clarke_subdiff}
Let $\psi: \mathbb{C}^{n} \rightarrow(-\infty,+\infty]$  be a locally Lipschitz function. The Clarke subdifferential of $\psi$  at $\xbf$  is defined as 
\[
\partial \psi(\xbf) := \Big\{\wbf \in \mathbb{C}^{n}\ \Big\vert \ \langle \wbf, \vbf \rangle \leq \limsup_{\zbf \rightarrow \xbf,\, t \downarrow 0} \frac{\psi(\zbf+t \vbf)-\psi(\zbf)}{t}, \ \ \forall\, \vbf \in \mathbb{C}^{n} \Big\}.
\]
\end{definition}

\begin{definition}[Clarke stationary point]
\label{def:clarke_cp}
For a locally Lipschitz function $\psi$, a point $\xbf \in \mathbb{C}^n$ is called a Clarke stationary point of $\psi$ if $\mathbf{0} \in \partial \psi(\xbf)$.
\end{definition}

Following Definition \ref{def:clarke_subdiff}, we find the Clarke subdifferential of $r$ at $\xbf$ to be
\begin{align}
\label{eq:r_subdiff}
    \partial r(\xbf) = \Big\{\sum_{k\in K_0(\xbf)}\nabla \qbf_{k}(\xbf)^{\top}  \wbf_{k} + \sum_{k \in K_1(\xbf)}\nabla \qbf_{k}(\xbf)^{\top}\frac{\qbf_{k}(\xbf)}{\|\qbf_{k}(\xbf)\|}
    \Big\vert \ \substack{\wbf_k \in \mathbb{C}^{d_2}, \\ \|\Pi(\wbf_{k}; \Ccal(\nabla \qbf_k(\xbf)))\|\leq 1}\Big\} ,  
\end{align}
where $K_0(\xbf):=K_0^0(\xbf)=\{k \in [d_1] \ | \ \|\qbf_k(\xbf) \|= 0 \}$, $K_1(\xbf):=[d_1] \setminus K_0(\xbf)$, and 
$\Pi(\wbf;\Ccal(\Abf))$ is the orthogonal projection of a vector $\wbf$ onto $\Ccal(\Abf)$ which stands for the column space of $\Abf$. The details to obtain \eqref{eq:r_subdiff} can be found in Lemma 3.1 in \cite{chen2021learnable}.

To prove the convergence properties and iteration complexity of Algorithm \ref{alg:elda}, the following assumptions are needed.

\begin{assumption}
We make the following assumptions:
\begin{itemize}
    \item[(A1)] The function $f$ is differentiable and $\nabla f$ is $L_f$-Lipschitz continuous on $\mathbb{C}^n$.
    \item[(A2)] For every $k \in [d_1]$, $\qbf_k$ is differentiable, $\nabla \qbf_k$ is $L_\qbf$-Lipschitz continuous; and there exists $M>0$ such that $\| \nabla \qbf_k (\xbf) \| \le M$ for all $\xbf$ and $k$.
    \item[(A3)] The objective function $\phi$ in \eqref{eq:formulation} is coercive, and $\phi^* := \min_{\xbf} \phi(\xbf) > -\infty$.
\end{itemize}
\end{assumption}
We remark that, by the design of $\phi$, $\gbf$ and $\hbf$, it is easy to have (A2) and (A3). (A1) can be weakened to $f$ being proper and locally Lipschitz on  $\mathbb{C}^n$. In this case, we can apply the smoothing procedure for both $f$ and $r$ to obtain the same convergence result as (A1) holds.

\begin{lemma}[Lemma 3.2 \cite{chen2021learnable}]\label{lemma1}
The gradient $\nabla \reps$ is $(\sqrt{d_1} L_\qbf+\frac{M^2}{\varepsilon})$-Lipschitz continuous.
Hence, $\nabla \phi_\varepsilon(\xbf)$ is $L_{\varepsilon}$-Lipschitz continuous with $L_{\varepsilon}:= L_f + \sqrt{d_1}L_\qbf + \frac{M^2}{\varepsilon} = O(\varepsilon^{-1})$.
\end{lemma}

We first consider the case where the smoothing parameter $\varepsilon>0$ is fixed and never reduced in Algorithm \ref{alg:elda}. This provides some important insights on the convergence properties of the algorithm in minimizing $\phi_{\varepsilon}$ with a given smoothing level $\varepsilon$.
\begin{lemma}\label{lemma2}
For any $\varepsilon_0>0$ and initial guess $\xbf_0$, let $\{\xbf_t\}$ be the sequence generated by Algorithm \ref{alg:elda} with $\sigma = \epstol = 0$ and $T=\infty$. For simplicity we denote $\varepsilon:=\varepsilon_0$. 
Then we have 
\begin{enumerate}
    \item The maximum number of required line searches $\ell_{\max}$ is 
    \begin{equation} \label{eq:lmax}
     \ell_{\max}=\Bigg\lfloor \frac{ \log \big[\bar{\alpha}(\frac{L_{\varepsilon}}{2}+ \frac{\eta_3}{\varepsilon})\big] }{ \log(1/\rho) }\Bigg\rfloor+1 \ .
    \end{equation}
    
    \item $\|\nabla \phi_{\varepsilon}(\xbf_t)\|\to 0$ as $t\to\infty$.
    
    \item For any $\eta>0$, there is 
    \begin{equation} \label{mink}
    \min\{t\in \mathbb N \ | \ \|\nabla \phi_{\varepsilon}(\xbf_{t+1})\|<\eta \}\le \max \Big\{\frac{\varepsilon (\frac{L_{\varepsilon}}{2}+\frac{\eta_3}{\varepsilon})^2}{\eta_3 \rho^2}, \frac{2}{\eta_1^2\eta_2} \Big\}
    \frac{\phi_{\varepsilon}(\xbf_0)-\phi^*+\frac{d_1\varepsilon}{2}}{\eta^2}.
    \end{equation}
\end{enumerate} 
 
\end{lemma}

\begin{proof}
For any iteration $t$, we know $\xbf_{t+1}=\ubf_{t+1}$ if $\ubf_{t+1}$ satisfies the conditions \eqref{eq:elda-condition1} and \eqref{eq:elda-condition2}, which implies
\begin{equation}\label{x=u}
\|\nabla \phi_{\varepsilon}(\ubf_{t+1})\|^2\le \frac{2(\phi_{\varepsilon}(\xbf_{t})-\phi_{\varepsilon}(\ubf_{t+1}))}{\eta_1^2\eta_2}.
\end{equation}
If either one of \eqref{eq:elda-condition1} and \eqref{eq:elda-condition2} is not met, then Algorithm \ref{alg:elda} computes $\vbf_{t+1}$ by the gradient descent method with the standard line search strategy (lines 7--13 in Algorithm \ref{alg:elda}). 
Let $\ell_t$ be the number of line searches required to meet condition \eqref{eq:v-condition} in the $t$th iteration, we have
\begin{equation}\label{v-12c}
\vbf_{t+1}=\xbf_t-\bar{\alpha} \rho^{\ell_t}\nabla \phi_{\varepsilon}(\xbf_{t}).
\end{equation}

Since $\nabla \phi_{\varepsilon}$ is $L_{\varepsilon}$-Lipschitz continuous, we have
\begin{align}
\phi_{\varepsilon}(\vbf_{t+1})&
\le \phi_{\varepsilon}(\xbf_t)+ \langle \nabla\phi_{\varepsilon}(\xbf_t)
, \, \vbf_{t+1}-\xbf_t \rangle+\frac{L_{\varepsilon}}{2}\|\vbf_{t+1}-\xbf_t\|^2\label{compare-phi-3}.
\end{align}
%
The combination of \eqref{v-12c} and \eqref{compare-phi-3} yields
\begin{equation}\label{v-decay}
\phi_{\varepsilon}(\vbf_{t+1}) \leq  \phi_{\varepsilon}(\xbf_{t})+ \Big(-\frac{1}{\bar{\alpha}\rho^{\ell_t}}+\frac{L_{\varepsilon}}{2} \Big) \|\vbf_{t+1}-\xbf_t\|^2. 
\end{equation}
Now it is easy to see that, if 
\begin{equation}\label{add-1}
-\frac{1}{\bar{\alpha}\rho^{\ell_t}}+\frac{L_{\varepsilon}}{2}\le -\frac{\eta_3}{\varepsilon},
\end{equation}
then the condition (\ref{eq:v-condition}) holds.
Hence, it is easy to check that the maximum number of line searches $\ell_{\max}$ required to have \eqref{eq:v-condition} is the $\ell_{\max}$ defined in \eqref{eq:lmax},
since this guarantees \eqref{add-1}.
This proves the first claim of this lemma. 

Moreover, from $\rho \in (0,1)$ and $\ell_t\leq \ell_{\max}$ for all $t$, we notice that
\begin{equation}\label{stepsize bound}
    \rho^{\ell_{t}}\geq  \rho^{\ell_{\max}} \ge
    \frac{\rho}{ \bar{\alpha}( \frac{L_{\varepsilon}}{2} + \frac{\eta_3}{\varepsilon} ) }.
    \end{equation}
where the second inequality is due to $\ell_{\max} -1 \le \frac{\log [\bar{\alpha}(\frac{L_{\varepsilon}}{2}+\frac{\eta_3}{\varepsilon})]}{\log (1/\rho)}$ from \eqref{eq:lmax}.

Now we prove the second claim. Due to (\ref{v-12c}), the condition (\ref{eq:v-condition}) can be rewritten as: 
\begin{align}\label{line-up-2}
\phi_{\varepsilon}(\vbf_{t+1})\le \phi_{\varepsilon}(\xbf_t)
-\frac{\eta_3}{\varepsilon}\bar{\alpha}^2\rho^{2\ell_t}\|\nabla \phi_{\varepsilon}(\xbf_t)\|^2.
\end{align}
Combining \eqref{stepsize bound} and \eqref{line-up-2}, we get
\begin{equation}\label{x=v}
 \|  \nabla \phi_{\varepsilon} (\xbf_t)\|^2 
 \leq \frac{\varepsilon}{\eta_3( \bar{\alpha} \rho^{\ell_{\max}} )^2} (\phi_{\varepsilon}(\xbf_{t}) - \phi_{\varepsilon}(\vbf_{t+1})).
\end{equation}

Hence, either $\xbf_{t+1}=\ubf_{t+1}$ or $\xbf_{t+1}=\vbf_{t+1}$ in the $t$th iteration, we have from \eqref{eq:elda-condition2} and \eqref{eq:v-condition} that
\begin{equation} \label{phi-decay}
    \phi_{\varepsilon}(\xbf_{t+1}) - \phi_{\varepsilon}(\xbf_t) \le - \min\{\eta_2/2, \eta_3/\varepsilon\} \| \xbf_{t+1} - \xbf_t \|^2
\end{equation}
Moreover, from \eqref{x=u} and \eqref{x=v}, we have
\begin{equation}\label{eq:diff3}
 \|  \nabla \phi_{\varepsilon} (\xbf_t)\|^2 \leq C_{\varepsilon} (\phi_{\varepsilon}(\xbf_{t}) - \phi_{\varepsilon}(\xbf_{t+1})),
\end{equation}
since
\begin{equation} \label{C}
0 < \max \Big\{ \frac{\varepsilon}{\eta_3 (\rho^{\ell_{\max}} \bar{\alpha})^ 2}, \frac{2}{\eta_1^2 \eta_2} \Big\} \le  \max  \Big\{\frac{\varepsilon (\frac{L_{\varepsilon}}{2}+\frac{\eta_3}{\varepsilon})^2}{\eta_3 \rho^2}, \frac{2}{\eta_1^2 \eta_2} \Big\} =: C_{\varepsilon},
\end{equation}
where we used \eqref{stepsize bound} to obtain the second inequality and the definition of $C_{\varepsilon}$ as above.

Summing up \eqref{eq:diff3} from $t=0$ to $t=T$ and using \eqref{eq:r_eps_bound}, we get
\begin{equation} \label{add-e-2}
    \sum_{t=0}^{T} \|  \nabla \phi_{\varepsilon} (\xbf_t) \| ^2 \leq C_{\varepsilon}(\phi_{\varepsilon}(\xbf_0) - \phi_{\varepsilon}(\xbf_{T+1})) \leq C_{\varepsilon} \Big(\phi_{\varepsilon}(\xbf_0) - \phi^* + \frac{d_1\varepsilon}{2} \Big).
\end{equation}
because $\phi_{\varepsilon}(\xbf_{T+1})) \ge \phi(\xbf_{T+1})) - \frac{d_1 \varepsilon}{2} \ge \phi^* - \frac{d_1 \varepsilon}{2}$.
Letting $T\to\infty$, we conclude that $\| \nabla \phi_{\varepsilon} (\xbf_t) \| \to 0$. The second claim is proved.

Now we prove the third claim. Let $\tau:=\min\{t\in\mathbb N\,|\, \|\nabla \phi_{\varepsilon}(\xbf_{t+1})\|< \eta\}$, then we know that $\|\phi_{\varepsilon}(\xbf_{t+1})\|\ge \eta$ for all $t\le \tau-1$. Thus from \eqref{add-e-2} we have
\[
\tau \eta^2\le \sum_{k=0}^{\tau-1}\|\nabla \phi_{\varepsilon}(\xbf_{t+1})\|^2=\sum_{k=1}^{\tau}\|\nabla \phi_{\varepsilon}(\xbf_t)\|^2\le C_{\varepsilon} \Big(\phi_{\varepsilon}(\xbf_0) - \phi^* + \frac{d_1\varepsilon}{2} \Big).
\]
This verifies the third claim. 
\end{proof}
\medskip

Now we are ready to show the convergence and iteration complexity of Algorithm 1 in solving \eqref{eq:formulation}.

\begin{theorem}\label{convergence-thm}  
Let $\{\xbf_t\}$ be the sequence generated by Algorithm \ref{alg:elda} with arbitrary initial guess $\xbf_0$. 
Denote $\{\tilde \xbf_l\} := \{\xbf_{t_l+1}\}$, which is the subsequence of $\{\xbf_t\}$ where the reduction criterion for $\varepsilon_t$ (line 15 in Algorithm \ref{alg:elda}) is met for $t=t_l$ and $l=1,2,\dots$. 
Then we have
\begin{enumerate}
    \item If $\epstol=0$ and $T=\infty$, then the sequence $\{\tilde{\xbf}_l\}$ has at least one accumulation point, and every accumulation point of $\{\tilde{\xbf}_l\}$ is a Clarke stationary point of \eqref{eq:formulation}. 

    \item For any $\epstol > 0$, there exist two constants $C_1>0$ and $C_2>0$ depending on $\eta_1$, $\eta_2$, $\eta_3$, $\rho$, $d_1$, $\sigma$, $\varepsilon_0$, and $\phi(\xbf_0)-\phi^*$, such that the number of iterations for the $l$th segment, i.e., $t_{l+1}-t_l$, has the following bound:
    \begin{equation}\label{eq:inner-length}
        t_{l+1}-t_l \le  C_1 \gamma^{-3l-2}+C_2\gamma^{-2l-2}.
    \end{equation}

    \item For any $\epstol>0$ and $T = \infty$, Algorithm \ref{alg:elda} terminates with $\hat{\ell}$ of $\varepsilon$-reductions (i.e., line 15 of Algorithm \ref{alg:elda} is met for $\hat{\ell}$ times and then the algorithm terminates), where $\hat{\ell}$ satisfies $\gamma^{-(\hat{\ell}-1)}\le \sigma \varepsilon_0 \epstol^{-1} = O(\epstol^{-1})$, and the total number of required iterations is upper bounded by 
    \begin{equation}\label{total iterations}
    \frac{C_1(\gamma^{-3(\hat{\ell}-1)}- \gamma^3)}{(1-\gamma^3)\gamma^2}+\frac{C_2(\gamma^{-2(\hat{\ell}-1)} - \gamma^2)}{(1-\gamma^2)\gamma^2}=O(\epstol^{-3}).
    \end{equation}
    
\end{enumerate}
\end{theorem}

\begin{proof}
    
By the definition of $r_{\varepsilon}(\xbf)$ in \eqref{eq:smoothedR}, we have
\begin{equation}
    r_{\varepsilon}(\xbf) + \frac{d_1\varepsilon}{2} = 
    \sum_{k \in K_0^{\varepsilon}(\xbf)} \Big(\frac{\| \qbf_k(\xbf) \|^2}{2\varepsilon} + \frac{\varepsilon}{2} \Big) + \sum_{k \in K_1^{\varepsilon}(\xbf)} \| \qbf_k(\xbf) \|.
\end{equation}   
Notice that, for any fixed $\xbf$, we have that $\frac{\| \qbf_k(\xbf) \|^2}{2\varepsilon} 
+ \frac{\varepsilon}{2}$ as a function of $\varepsilon$ is non-decreasing in $\varepsilon$ for any $k \in K_0^{\varepsilon}(\xbf)$ where $\|\qbf_k(\xbf)\| \le \varepsilon $. 
Therefore, $\phi_{\varepsilon}(\xbf)+\frac{d_1 \varepsilon}{2}=f(\xbf) + r_{\varepsilon}(\xbf) + \frac{d_1\varepsilon}{2}$ is non-decreasing in $\varepsilon$ for any fixed $\xbf\in \mathbb{C}^n$. 
This leads to the following monotonicity estimate:
\begin{equation}\label{eq:phi_decay}
    \phi_{\varepsilon_{t+1}}(\xbf_{t+1}) + \frac{d_1\varepsilon_{t+1}}{2} \le \phi_{\varepsilon_{t}}(\xbf_{t+1}) + \frac{d_1 \varepsilon_{t}}{2} \le \phi_{\varepsilon_{t}}(\xbf_{t}) + \frac{d_1 \varepsilon_{t}}{2},
\end{equation}
where the second inequality is due to $\phi_{\varepsilon_{t}}(\xbf_{t+1}) \le \phi_{\varepsilon_{t}}(\xbf_{t})$ from \eqref{phi-decay}. 
Therefore, \eqref{eq:phi_decay} implies a chain of inequalities
\begin{equation*}
\phi(\xbf_t) \le \phi_{\varepsilon_t}(\xbf_t) + \frac{d_1\varepsilon_t}{2} \le \cdots \le \phi_{\varepsilon_0}(\xbf_0) + \frac{d_1\varepsilon_0}{2} < \infty.
\end{equation*}
for all $t$, where the first inequality is due to \eqref{eq:r_eps_bound}.
Since $\phi$ is coercive, we know the sequence $\{\xbf_t\}$ is in a bounded domain, and so does $\{\xbf_{t_l+1}\}$. 
Hence, $\{\xbf_{t_l+1}\}$ has at least one accumulation point. 
Now, let $\{\xbf_{t_{l_j}+1}\}$ be any one of the convergent subsequences of $\{\xbf_{t_l+1}\}$ with limit $\bar{\xbf}$.
Since $\xbf_{t_l+1}$ satisfies the reduction criterion in line 15 of Algorithm \ref{alg:elda}, for the subsequence $\{\xbf_{t_{l_j}+1}\}$, we have that, as $j \to \infty$,
\begin{equation} \label{gradphi}
    \| \nabla \phi_{\varepsilon_{t_{l_j}}} (\xbf_{t_{l_j}+1}) \| \le \sigma \gamma \varepsilon_{t_{l_j}+1} = \sigma \varepsilon_0 \gamma^{{l_j}+1} \to 0.
\end{equation}

Next, we need to show that
\begin{equation} \label{lim grad smoothPhi}
    \lim_{j \rightarrow \infty}\nabla \phi_{\varepsilon_{t_{l_j}}} (\xbf_{t_{l_j}+1})\in \partial \phi (\bar{\xbf}).
\end{equation}
To this end, from \eqref{eq:r_subdiff}, we know the Clarke subdifferential of $\phi$ at $\bar{\xbf}$ is given as
\begin{align}
\partial \phi(\bar \xbf) = \bigg\{\nabla f(\bar \xbf) +\sum_{k\in K_0(\bar{\xbf})}\nabla \qbf_k(\bar\xbf)^{\top}  \wbf_k + \sum_{k \in K_1(\bar{\xbf})}\nabla \qbf_k(\bar\xbf)^{\top}\frac{\qbf_k(\bar\xbf)}{\|\qbf_k(\bar\xbf)\|} \ \bigg\vert \ \substack{\wbf_k \in \mathbb{C}^{d_2}, \\ \|\Pi(\wbf_k; \Ccal(\nabla \qbf_k(\bar\xbf)))\|\leq 1} \bigg\}.  \label{eq:phi_subdiff}
\end{align} 
On the other hand, we also know that $\nabla r_{\varepsilon_{t_{l_j}}} (\xbf_{t_{l_j}+1})$ is given by
\begin{equation}
\label{eq:d_r_epsj}
\nabla r_{\varepsilon_{t_{l_j}}} (\xbf_{t_{l_j}+1}) =
\sum_{k \in Q_0^j} \nabla \qbf_k(\xbf_{t_{l_j}+1})^{\top}\frac{\qbf_k(\xbf_{t_{l_j}+1})}{\varepsilon_{t_{l_j}}} + \sum_{k \in Q_1^j} \nabla \qbf_k(\xbf_{t_{l_j}+1})^{\top} \frac{\qbf_k(\xbf_{t_{l_j}+1})}{\| \qbf_k(\xbf_{t_{l_j}+1}) \|}. 
\end{equation}
where we introduce the following simpler notations for every $j$:
\begin{equation}
    Q_0^j := K_0^{\varepsilon_{t_{l_j}}}(\xbf_{t_{l_j}+1}), \qquad Q_1^j := [d_1] \setminus Q_0^j,
\end{equation}
and the definitions of $K_0^{\varepsilon}$ and $K_1^{\varepsilon}$ are given in \eqref{eq:def-K01}.
Furthermore, we define
\begin{equation}
    \label{eq:epshat-xbar}
    \hat{\varepsilon}(\bar{\xbf}) := \min \{ \|\qbf_k(\bar{\xbf})\| \ | \ \|\qbf_k(\bar{\xbf})\| > 0, \ k \in [d_1] \}.
\end{equation}
Notice that $\hat{\varepsilon}(\bar{\xbf}) > 0$ since the set minimized on the right-hand side of \eqref{eq:epshat-xbar} is a finite set of positive numbers.
Since $\varepsilon_{t_{l_j}} \to 0$ and $\xbf_{t_{l_j}+1}\rightarrow \bar{\xbf} $ as $j \to \infty$, and $\qbf_k$ is continuous for every $k \in [d_1]$, we know there exists $\hat{j} \in \mathbb{N}$, such that for all $j \ge \hat{j}$ there are
\begin{align}
    \|\qbf_k(\xbf_{t_{l_j}+1}) - \qbf_k(\bar{\xbf}) \| & \le \frac{\hat{\varepsilon}(\bar{\xbf})}{3} , \quad \forall\, k \in [d_1] , \label{eq:qkxlj-qkx-bound} \\
    0 < \varepsilon_{t_{l_j}} & \le \frac{\hat{\varepsilon}(\bar{\xbf})}{3} .\label{eq:epstlj}
\end{align}
Therefore, for any $j \ge \hat{j}$ and $k \in K_1(\bar{\xbf})$, there is
\begin{equation}
\label{eq:qj-epsj}
\|\qbf_k(\xbf_{t_{l_j}+1})\| \ge \|\qbf_k(\bar{\xbf})\| - \frac{\hat{\varepsilon}(\bar{\xbf})}{3} \ge \frac{2\hat{\varepsilon}(\bar{\xbf})}{3} > \frac{\hat{\varepsilon}(\bar{\xbf})}{3} \ge \varepsilon_{t_{l_j}}
\end{equation}
where the first inequality is due to \eqref{eq:qkxlj-qkx-bound}, the second due to \eqref{eq:epshat-xbar}, and the third due to \eqref{eq:epstlj}. Therefore \eqref{eq:qj-epsj} implies that $k \in Q_1^j$. Hence $K_1(\bar{\xbf}) \subset Q_1^j$ when $j \ge \hat{j}$.
In this case, we have 
\[
[d_1] = (K_0(\bar{\xbf}) \cap Q_0^j)\ \cup\ (K_0(\bar{\xbf}) \cap Q_1^j)\ \cup\ K_1(\bar{\xbf}),
\]
which means that $[d_1]$ is a disjoint union of these three sets.
Therefore, for all $j \ge \hat{j}$, we can rewrite \eqref{eq:d_r_epsj} as
\begin{align}
    \nabla r_{\varepsilon_{t_{l_j}}} (\xbf_{t_{l_j}+1}) & =
\sum_{k \in K_0(\bar{\xbf}) \cap Q_0^j} \nabla \qbf_k(\xbf_{t_{l_j}+1})^{\top}\frac{\qbf_k(\xbf_{t_{l_j}+1})}{\varepsilon_{t_{l_j}}} + \sum_{k \in K_0(\bar{\xbf}) \cap Q_1^j} \nabla \qbf_k(\xbf_{t_{l_j}+1})^{\top}\frac{\qbf_k(\xbf_{t_{l_j}+1})}{\| \qbf_k(\xbf_{t_{l_j}+1}) \|} \nonumber \\
& \qquad + \sum_{k \in K_1(\bar{\xbf})} \nabla \qbf_k(\xbf_{t_{l_j}+1})^{\top} \frac{\qbf_k(\xbf_{t_{l_j}+1})}{\| \qbf_k(\xbf_{t_{l_j}+1}) \|}. \label{eq:d_r_epsj_new}
\end{align}
Since $\qbf_k$ and $\nabla \qbf_k$ are continuous for all $k$, and $\xbf_{t_{l_j}+1} \to \bar{\xbf}$ as $j \to \infty$, we know the third term on the right-hand side of \eqref{eq:d_r_epsj_new} converges as $j\to \infty$:
\begin{equation}
    \label{eq:d_r_term3}
    \sum_{k \in K_1(\bar{\xbf})} \nabla \qbf_k(\xbf_{t_{l_j}+1})^{\top} \frac{\qbf_k(\xbf_{t_{l_j}+1})}{\| \qbf_k(\xbf_{t_{l_j}+1}) \|} \quad \to \quad \sum_{k \in K_1(\bar{\xbf})}\nabla \qbf_k(\bar\xbf)^{\top}\frac{\qbf_k(\bar\xbf)}{\|\qbf_k(\bar\xbf)\|}.
\end{equation}
Now we check the convergence property of the sum of the first two terms on the right-hand side of \eqref{eq:d_r_epsj_new}, which we denote by $\sbf_j$, as $j \to \infty$. 
We notice that 
\begin{equation}
\label{eq:multiple-bound1}
\frac{\|\qbf_k(\xbf_{t_{l_j}+1})\|}{\varepsilon_{t_{l_j}}} \le 1, \qquad \Bigg\|\frac{\qbf_k(\xbf_{t_{l_j}+1})}{\| \qbf_k(\xbf_{t_{l_j}+1}) \|} \Bigg\| = 1    
\end{equation}
for all $k \in [d_1]$ and $j \in \mathbb{N}$.
Since $\nabla \qbf_k$ is continuous, we know $\nabla \qbf_k (\xbf_{t_{l_j}+1}) \to \nabla \qbf_k(\bar{\xbf})$ as $j\to \infty$. 
Moreover, for any $\delta \ge 0$, we denote
\begin{equation}
\label{eq:B-delta}
    B_{\delta}(\bar{\xbf}):= \Big\{ \sum_{k \in K_0(\bar{\xbf})} \nabla \qbf_k(\bar{\xbf})^{\top} \wbf_k \ \Big\vert \ \|\wbf_k\| \le 1+\delta, \ \wbf_k \in \mathbb{C}^{d_2}, \ \forall\, k \in K_0(\bar{\xbf}) \Big\} \subset \mathbb{C}^{n}.
\end{equation}
Then for any $\delta>0$, we know $\sbf_j$ will fall into $B_{\delta}(\bar{\xbf})$ for all sufficiently large $j$ due to \eqref{eq:multiple-bound1}.
%
%
Therefore, $\sbf_j$ must converge to the set $B_0(\bar{\xbf}) = \cap_{\delta>0} B_{\delta}(\bar{\xbf})$.
%
%
We also notice that $\| \Pi(\wbf_k; \Ccal(\nabla \qbf_k(\xbf))) \| \le \|\wbf_k\| \le 1$ for any admissible $\wbf_k$ in $B_0(\bar{\xbf})$, since $\Pi$ is an orthogonal projection of $\wbf_k$ onto $\Ccal(\nabla \qbf_k(\xbf))$ which is a linear subspace of $\mathbb{C}^{d_2}$.
Therefore, we have
\[
B_0(\bar{\xbf}) \subset \Big\{ \sum_{k \in K_0(\bar{\xbf})} \nabla \qbf_k(\bar{\xbf})^{\top} \wbf_k \ \Big\vert \ \|  \Pi(\wbf_k; \Ccal(\nabla \qbf_k(\xbf))) \| \le 1, \ \wbf_k \in \mathbb{C}^{d_2}, \ \forall k \in K_0(\bar{\xbf}) \Big\}.
\]
Notice that $B_0(\bar{\xbf})$ is closed and bounded, and hence compact in $\mathbb{C}^{n}$.
Combining with \eqref{eq:d_r_epsj_new}, \eqref{eq:d_r_term3} and $\nabla f(\xbf_{t_{l_j}+1}) \to \nabla f(\bar{\xbf})$ as $\nabla f$ is continuous, we know $\nabla \phi_{\varepsilon_{t_{l_j}}} (\xbf_{t_{l_j}+1}) = \nabla f(\xbf_{t_{l_j}+1}) + \nabla r_{\varepsilon_{t_{l_j}}} (\xbf_{t_{l_j}+1})$ converges to a point in
\[
\nabla f(\bar{\xbf}) + B_0(\bar{\xbf}) + \sum_{k \in K_1(\bar{\xbf})}\nabla \qbf_k(\bar\xbf)^{\top}\frac{\qbf_k(\bar\xbf)}{\|\qbf_k(\bar\xbf)\|} \subset \partial \phi(\bar{\xbf})
\]
as $j \to \infty$, which verifies \eqref{lim grad smoothPhi}.
As $\partial \phi(\bar{\xbf})$ is closed, from \eqref{gradphi} and \eqref{lim grad smoothPhi}, we conclude that $\mathbf{0} \in \partial \phi(\bar{\xbf})$. 
Hence $\bar{\xbf}$ is a  Clarke stationary point of $\phi (\xbf)$. The first claim is proved.

Next, given $\epstol> 0$, we need to estimate $t_{l+1}-t_l$. 
By \eqref{mink} in Lemma \ref{lemma2} with its $\varepsilon$, $\eta$ and initial set to $\varepsilon_{t_l}=\varepsilon_0 \gamma^l$, $\sigma \gamma \varepsilon_{t_l+1}=\sigma\varepsilon_0\gamma^{l+1}$, and $\xbf_{t_l+1}$, respectively, we have
\begin{equation}
\label{eq:tl1-tl}
    t_{l+1}-t_l\le \max \Big\{\frac{\varepsilon_0 \gamma^l(\frac{L_{\varepsilon_0 \gamma^l}}{2}+\frac{\eta_3}{\varepsilon_0\gamma^l})^2}{\eta_3 \rho^2}, \frac{2}{\eta_1^2\eta_2} \Big\}
    \frac{\phi_{\varepsilon_0 \gamma^l}(\xbf_0)-\phi^*+\frac{d_1\varepsilon_0 \gamma^l}{2}}{(\sigma\varepsilon_0\gamma^{l+1})^2}.
\end{equation}
Recall that $L_{\varepsilon_{t_l}} = L_{\varepsilon_0 \gamma^l}=O((\varepsilon_0 \gamma^l)^{-1})$ from Lemma \ref{lemma1}, and $\phi_{\varepsilon_0 \gamma^l}(\xbf_0)\le \phi(\xbf_0)$ from \eqref{eq:r_eps_bound}, we know there exist constants $C_1, C_2>0$ depending on $\eta_1$, $\eta_2$, $\eta_3$, $\rho$, $d_1$, $\sigma$, $\varepsilon_0$ and $\phi(\xbf_0)-\phi^*$ only, such that
\begin{equation}
 t_{l+1}-t_l\le   \gamma^{-l}(C_1+C_2 \gamma^l) \gamma^{-2(l+1)}=C_1 \gamma^{-3l-2}+C_2\gamma^{-2l-2},
\end{equation}
because the max term in \eqref{eq:tl1-tl} is of $O(\gamma^{-l})$. This verifies the second claim.

To prove the last claim, we need to compute the total number of $\varepsilon$-reductions (i.e., line 15 is met in Algorithm \ref{alg:elda}) required to reduce $\varepsilon_0$ to $\epstol$.
Let $\hat{\ell}$ be the number of such reductions, then $\sigma \varepsilon_0\gamma^{\hat{\ell}-1}\ge \epstol$. Thus, 
\begin{equation}\label{est-l-0}
\gamma^{-(\hat{\ell}-1)}\le \sigma \varepsilon_0 \epstol^{-1}.
\end{equation}
From \eqref{eq:inner-length}, we have
\begin{align*}
\sum_{l=0}^{\hat{\ell}-1}(t_{l+1}-t_l) \le \sum_{l=1}^{\hat{\ell}-1} (C_1 \gamma^{-3l-2}+C_2\gamma^{-2l-2}) \le \frac{C_1(\gamma^{-3(\hat{\ell}-1)}- \gamma^3)}{(1-\gamma^3)\gamma^2}+\frac{C_2(\gamma^{-2(\hat{\ell}-1)} - \gamma^2)}{(1-\gamma^2)\gamma^2},
\end{align*}
which is $O(\epstol^{-3})$ due to \eqref{est-l-0}.
Therefore, the last claim is proved.
\end{proof}

\subsection{Unrolling network of modified ELDA for TL}
While the modified ELDA Algorithm \ref{alg:elda} is to solve general non-smooth non-convex optimization problems of form \eqref{eq:formulation}, we need to adopt it in our TL approach to learn the feature-extractor and adapters described in \eqref{eq:solve-g} and \eqref{eq:solve-h}. 
To this end, we use the idea of unrolling networks by replacing the lower-level optimization problem solver with a $T$-phase network, where the $t$th phase is exactly the $t$th iteration of Algorithm \ref{alg:elda}. Here $T$ is a manually set phase number, which is usually between 10 and 20, and we will discuss how to determine it in Section \ref{subsec:ablation}. 
This network has measurement $\ybf$ as input and reconstruction $\xbf$ as output, where $\xbf$ is a function of the parameters of $\gbf$ and $\hbf_i/\hat{\hbf}_j$. 
Then $\xbf$ is plugged into the upper-level problem, i.e., \eqref{eq:solve-g-upper} and \eqref{eq:solve-h-upper}, and the problems reduce to minimizing the loss function of the parameters of $\gbf$ and $\hbf_i/\hat{\hbf}_j$. We solve the minimization problem using ADAM \cite{kingma2015adam} in our experiments on the data sets to be specified later.
This $T$-phase reconstruction unrolling network following the universal learnable descent algorithm \ref{alg:elda} is called U-LDA in our experiments below.

\section{Experimental results}\label{sec:experiment}

\subsection{Implementation details}
\label{subsec:implementation}

We provide some details of the U-LDA implementation, including: the network architectures of the feature-extractor and adapters, an effective method to initialize these networks, and a data augmentation method to improve data exploitation.

\paragraph{Network architecture} U-LDA is a $T$-phase unrolling network based on Algorithm \ref{alg:elda}, i.e., the $t$th phase of U-LDA exactly follows the computation of one iteration in Algorithm \ref{alg:elda}, for $t=0,1,\dots,T$. The differences when training the U-LDA in the two steps described in Section \ref{sec:method} are that both $\gbf$ and $\hbf_i$ are trained in Step 1, while $\gbf$ is fixed and only $\hat{\hbf}_j$'s are trained in Step 2. 
In our U-LDA implementation, the universal feature-extractor $\gbf$ is a CNN consisting of four layers of complex-valued convolutions, each followed by a smoothed ReLU activation \cite{chen2021learnable}. Each layer has complex convolution and uses a $3\times3$ kernel with 16 channels. All the adapters $\hbf_i$'s and $\hat{\hbf}_j$'s have the same architecture, which is a CNN with one complex convolutional layer with $3\times3$ kernel and 16 channels.

\paragraph{Initialization of the universal feature-extractor} 
In our numerical tests, we found that a proper initialization of the feature-extractor $\gbf$ can significantly improve training and yield better reconstruction quality. 
Specifically, we train $\gbf$ without the adapters $\hbf_i$'s in the problem \eqref{eq:solve-g} using the LDA method \cite{chen2021learnable} for each data set $D_i$. This results in a $\gbf_i$ for every $i \in [I]$. Since $\gbf$ and $\gbf_i$'s have the same architecture, we set the initial parameters of $\gbf$ in the training in Step 1 as the average of the counterpart parameters of $\gbf_i$'s. 
Empirical results demonstrate that solution quality becomes noticeably better than using typical random parameter initialization.

\paragraph{Artificial undersampling for data augmentation}
We propose an artificial undersampling strategy to enhance data exploitation in handling small data set $\Dhat$. 
Specifically, we further undersample the already under-sampled $k$-space data $\ybf$. In the case of MR image reconstruction, for instance, we keep most of the low frequency $k$-space data while significantly undersample the remaining high frequency data of $\ybf$, such that the total amount of data is half of the original under-sampled data. 
Then we train U-LDA on one of the two data sets, use the trained parameters as the initial to train U-LDA on the other data set, and finally, the newly trained parameters as the initial to train U-LDA on the original data. This strategy has a data augmentation effect similar to bootstrapping to exploit more information from small data, and as a consequence improves the training quality in practice.

\subsection{Numerical results}
\label{subsec:results}
To evaluate the performance of the proposed method, we conduct numerical experiments in three different TL settings in MRI reconstruction. 
\emph{Cross-anatomy}: transferring knowledge learned from large image datasets of some anatomy regions to reconstruct images of different anatomies with limited data; 
\emph{Cross-sampling-rate}: transferring knowledge learned from large datasets with several given sampling ratios to reconstruct images of different sampling ratios; 
\emph{Cross-modality}: transferring knowledge learned from large natural image datasets to reconstruct MR images. 

{We use Cartesian sampling masks to artificially undersample the $k$-space measurements in all experiments. All $k$-space data are measured in single-coil setting. 
In all but the cross-sampling-ratio experiment, we use a fixed Cartesian mask of 20\% sampling ratio.
In the cross-sampling-ratio experiment, we use several Cartesian masks with different sampling ratios, i.e., 10\%, 15\% 20\%, 25\%, and 30\%.
In either case, the Cartesian mask acquires all the lowest 5\% frequencies.
The remainder is acquired by following a Gaussian weighted random sampling, such that low frequency coefficients are more likely to be sampled while the total amount of acquired coefficients is equal to the preset ratio. 
}

{In our experiments, we also use several state-of-the-art TL methods for comparison. Specifically, we compare with recent TL methods Universal MRI (U-MRI)~\cite{liu2021universal} and the meta-learning (Meta) \cite{bian2021optimization-based}.
In addition, we include DnCn~\cite{schlemper2018deep} and LDA~\cite{chen2021learnable} which are non-TL methods for baseline comparison. DnCn and LDA are trained and tested on each dataset separately, and then further trained on the testing datasets using a smaller learning rate and fewer epochs. 
For DnCn, we adopt the same network architecture as the independent models used in~\cite{liu2021universal}. The model is trained for 200 epochs 
using Adam optimizer with learning rate $10^{-4}$ and other parameters as default. 
For U-MRI, we follow the training and transfer procedures reported in~\cite{liu2021universal}. Since \cite{liu2021universal} does not specify the number of epochs for regular and transfer training, we perform regular training for 200 epochs and apply early stopping during the transfer phase. We use the Adam optimizer with learning rate $10^{-4}$ for regular training and $10^{-5}$ for transfer training. 
For Meta, we apply the warm-up training strategy adopted in both~\cite{bian2021optimization-based,chen2021learnable}. Specifically, we perform warm-up training for 100 epochs every two phases, using a decayed learning rate starting from $10^{-4}$ with a decay factor of $0.8$ applied every two phases. In Meta, a validation step of 5 inner-loop updates is performed for each batch.
For LDA, we also employ the same warm-up training strategy, with 100 epochs of training every two phases.}

{We further fine-tune the UNet~\cite{arshad2021transfer} and the HUMUS-Net~\cite{fabian2022humus-net} as alternative TL methods.
By fine-tuning, we mean that both networks are pre-trained on all available data and then fine-tuned on the target data set $\Dhat$. Fine-tuning is performed using Adam with learning rate $10^{-5}$, batch size of 4, and 100 epochs for UNet and 50 epochs for HUMUS-Net. These settings are selected as they yield near-optimal PSNR values, and further training does not show noticeable improvements. }

{For reference purpose, images obtained by full $k$-space and undersampled $k$-space measurements, referred to as ``True" and ``Zero-filling", respectively, are included in visual comparisons below.}

{Our experiments are implemented in PyTorch package and conducted in the same computer with one NVIDIA A100 GPU. The code will be released to the public upon acceptance of this work.}

\subsubsection{Cross-anatomy transfer}

\label{subsubsec:cross-anatomy}
MRI reconstruction often suffer the issue of data limitation in certain anatomical regions. The experiments here are motivated by the potential of our method in learning knowledge from large MRI datasets including MR images of anatomical regions different from the target ones, and transferring to the latter which have limited data.

We first evaluate the proposed method on the fastMRI dataset \cite{knoll2020fastmri,zbontar2018fastmri}.
We retrieve 400 2D brain image slices to form $D_1$, and 400 knee images to form $D_2$. All images are grayscale with resolution $256\times 256$.
We also retrieve another 200 images for each of these two anatomies for testing purpose, i.e., to evaluate the performance of the trained $\hbf_i \circ \gbf$ on the reconstruction of these images.
We retrieve 100 grayscale cardiac images of resolution $256\times 256$ to form $\hat{D}_1$ and another 100 prostate images to form $\hat{D}_2$.

We use the same Cartesian sampling mask with 20\% sampling ratio on the $k$-spaces of all these images. After training the adapters $\hat{\hbf}_j$, we also randomly sample 200 images from the cardiac and prostate data excluding those in $\hat{D}_1$ and $\hat{D}_2$ to test the transfer ability of $\hat{\hbf}_j \circ \gbf$.
The average PSNR and SSIM of reconstructions by all compared methods on these four anatomies are given in Table~\ref{tab:MRI-anatomy}, where anatomies marked by $\dagger$ are those with limited image data for training and we aim at improving the reconstruction by transferring knowledge from other anatomies. 
In Table~\ref{tab:MRI-anatomy}, we see that the proposed U-LDA method demonstrates significant improvements in reconstruction quality by transferring knowledge learned from large brain and knee datasets to reconstruct cardiac and prostate images with small training datasets. By comparing LDA (a non-TL reconstruction method) and U-LDA, we also find that the knowledge transfer ability of U-LDA helps to improve the reconstruction quality across different anatomies. This further suggests the efficacy of knowledge transfer in deep learning based image reconstruction.
We also randomly select one image from each of the four test anatomies and show the images obtained by different compared methods in Figure~\ref{fig:cross-anatomy}, which also indicate improved reconstruction quality using the proposed U-LDA. 
\begin{table}[t]
\caption{Average PSNR and SSIM of reconstructed images on four different anatomies in Section \ref{subsubsec:cross-anatomy}. The anatomies marked by $\dagger$ (cardiac and prostate) have small training datasets. A fixed Cartesian mask of 20\% sampling ratio is used to undersample the $k$-spaces of all images.}
    \centering
    \scalebox{0.9}{
    \begin{tabular}{ccccccccc}
    \toprule
    \multirow{3}{*}{Method} & \multicolumn{4}{c}{PSNR (dB)} & \multicolumn{4}{c}{SSIM (\%)}\\
    \cmidrule(l{1em}r{1em}){2-5} \cmidrule(l{1em}r{1em}){6-9}
    & Brain & Knee & $\text{Cardiac}^\dagger$ & $\text{Prostate}^\dagger$ & Brain & Knee & $\text{Cardiac}^\dagger$ & $\text{Prostate}^\dagger$\\
    \addlinespace
    \cmidrule{1-9}
    DnCn
    ~\cite{schlemper2018deep} & 36.02 & 32.17 & 33.91 & 31.32 & 94.36 & 80.50 & 91.66 & 85.59\\
    LDA
    ~\cite{chen2021learnable} & 37.38 & 32.38 & 34.78 & 32.92 & 95.42 & 80.63 & 93.16 & 89.10 \\
    UNet
    ~\cite{arshad2021transfer} & 34.05 & 31.78 & 32.05 & 30.74 & 92.33 & 79.72 & 88.50 & 84.49\\ 
    U-MRI \cite{liu2021universal} & 36.40 & 32.30 & 34.85 & 31.69 & 94.67 & 80.67 & 92.88 & 86.88 \\
    Meta \cite{bian2021optimization-based} & 36.32 & 32.11 & 35.61 & 32.23 & 94.67 & 79.81 & 93.95 & 87.77\\
    HUMUS-Net
    ~\cite{fabian2022humus-net} & 36.29 & 32.39 & 35.22 & 32.62 & 94.52 & 80.91 & 93.39 & 88.55\\
    U-LDA (Ours) & \textbf{37.88} & \textbf{32.47} & \textbf{37.26} & \textbf{33.31} & \textbf{95.81} & \textbf{80.84} & \textbf{95.59} & \textbf{89.90} \\
    \bottomrule
    \end{tabular}}
    \label{tab:MRI-anatomy}
\end{table}
\begin{figure}[h!]
    \centering
    \begin{subfigure}[b]{0.16\textwidth}
        \includegraphics[width=\linewidth]{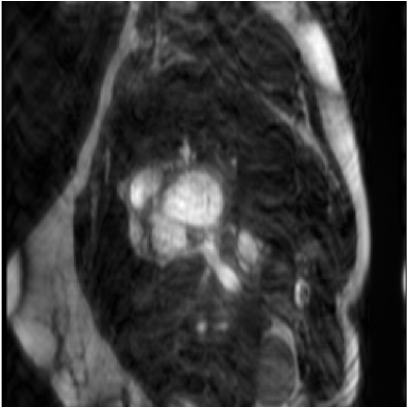}
    \end{subfigure}
    \hfill
    \begin{subfigure}[b]{0.16\textwidth}
        \includegraphics[width=\linewidth]{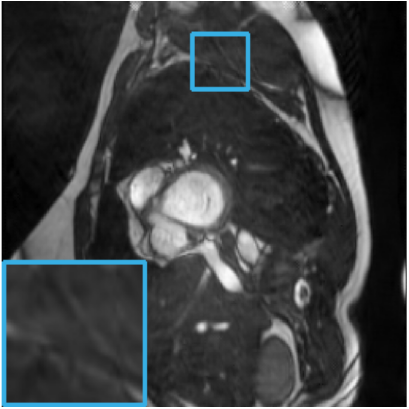}
    \end{subfigure}
    \hfill
    \begin{subfigure}[b]{0.16\textwidth}
        \includegraphics[width=\linewidth]{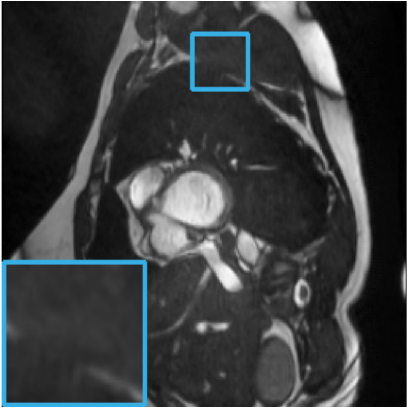}
    \end{subfigure}
    \hfill
    \begin{subfigure}[b]{0.16\textwidth}
        \includegraphics[width=\linewidth]{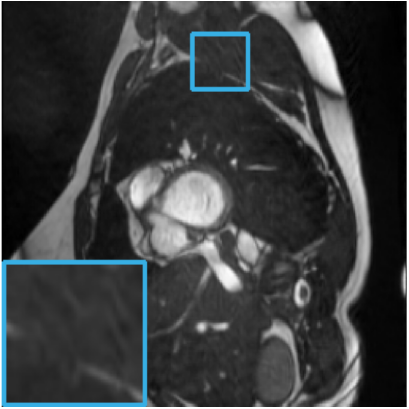}
    \end{subfigure}
    \hfill
    \begin{subfigure}[b]{0.16\textwidth}
        \includegraphics[width=\linewidth]{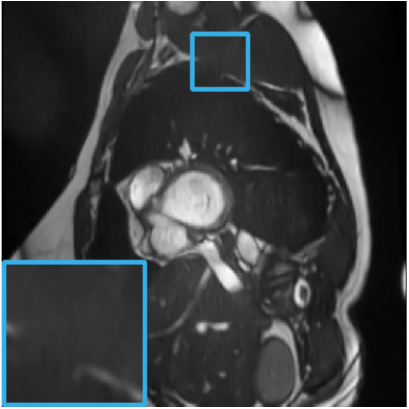}
    \end{subfigure}
    \hfill
    \begin{subfigure}[b]{0.16\textwidth}
        \includegraphics[width=\linewidth]{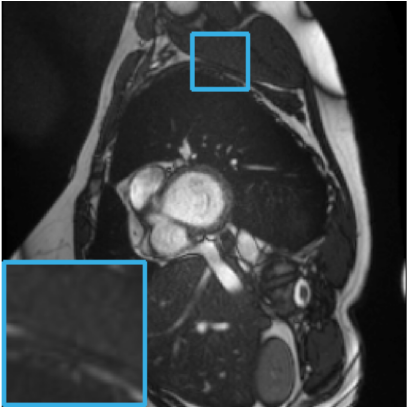}
    \end{subfigure}
    \begin{subfigure}[b]{0.16\textwidth}
        \includegraphics[width=\linewidth]{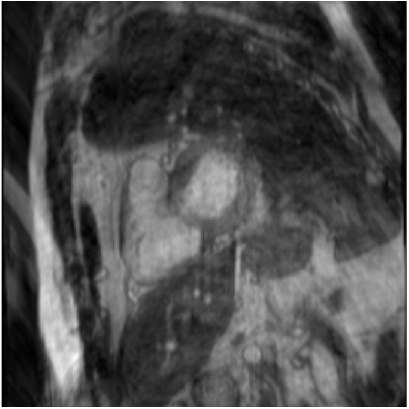}
    \end{subfigure}
    \hfill
    \begin{subfigure}[b]{0.16\textwidth}
        \includegraphics[width=\linewidth]{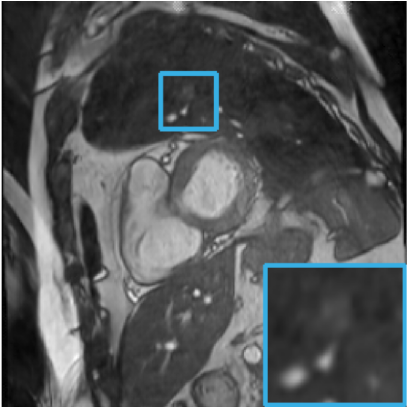}
    \end{subfigure}
    \hfill
    \begin{subfigure}[b]{0.16\textwidth}
        \includegraphics[width=\linewidth]{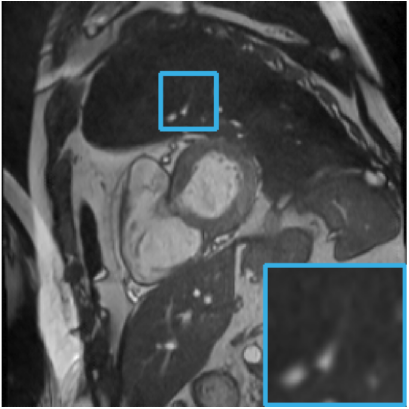}
    \end{subfigure}
    \hfill
    \begin{subfigure}[b]{0.16\textwidth}
        \includegraphics[width=\linewidth]{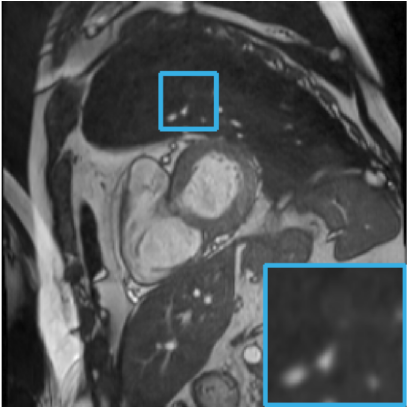}
    \end{subfigure}
    \hfill
    \begin{subfigure}[b]{0.16\textwidth}
        \includegraphics[width=\linewidth]{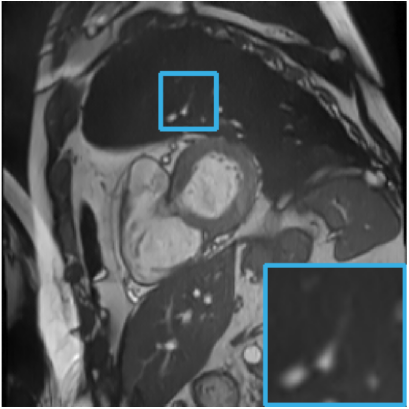}
    \end{subfigure}
    \hfill
    \begin{subfigure}[b]{0.16\textwidth}
        \includegraphics[width=\linewidth]{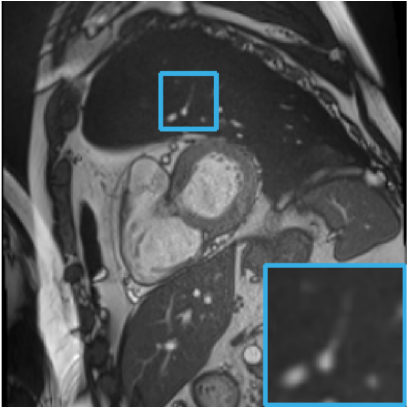}
    \end{subfigure}
    \begin{subfigure}[b]{0.16\textwidth}
        \includegraphics[width=\linewidth]{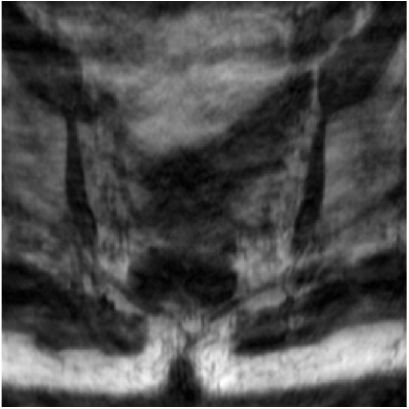}
    \end{subfigure}
    \hfill
    \begin{subfigure}[b]{0.16\textwidth}
    \includegraphics[width=\linewidth]{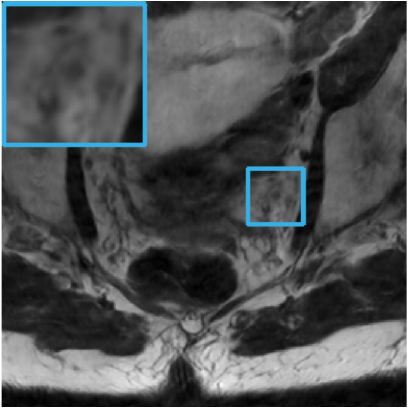}
    \end{subfigure}
    \hfill
    \begin{subfigure}[b]{0.16\textwidth}
    \includegraphics[width=\linewidth]{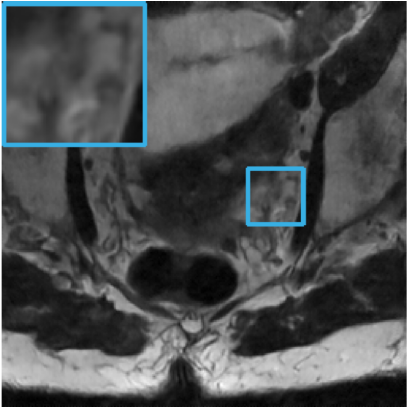}
    \end{subfigure}
    \hfill
    \begin{subfigure}[b]{0.16\textwidth}
    \includegraphics[width=\linewidth]{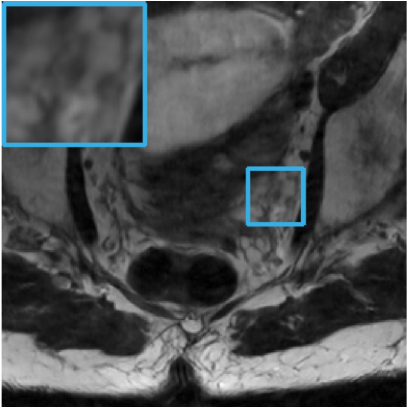}
    \end{subfigure}
    \hfill
    \begin{subfigure}[b]{0.16\textwidth}
    \includegraphics[width=\linewidth]{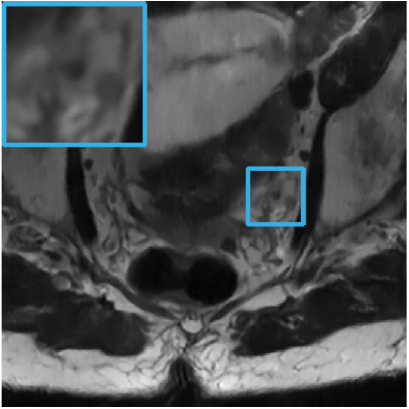}
    \end{subfigure}
    \hfill
    \begin{subfigure}[b]{0.16\textwidth}
    \includegraphics[width=\linewidth]{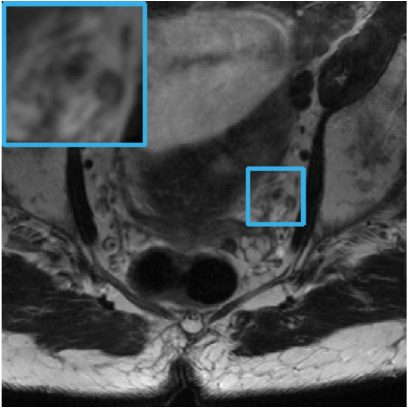}
    \end{subfigure}
    \begin{subfigure}[b]{0.16\textwidth}
        \includegraphics[width=\linewidth]{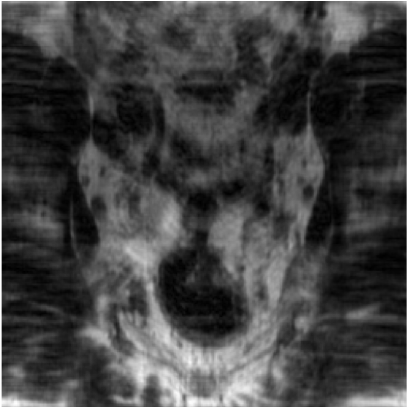}
        \caption{Zero-filling}
    \end{subfigure}
    \hfill
    \begin{subfigure}[b]{0.16\textwidth}
    \includegraphics[width=\linewidth]{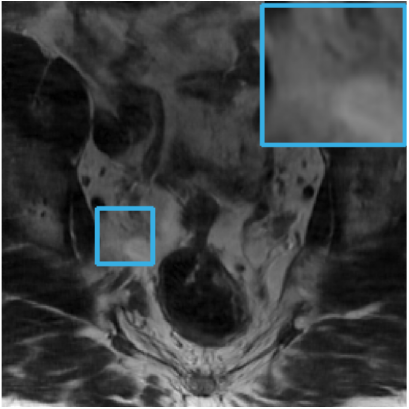}
    \caption{UNet}
    \end{subfigure}
    \hfill
    \begin{subfigure}[b]{0.16\textwidth}
    \includegraphics[width=\linewidth]{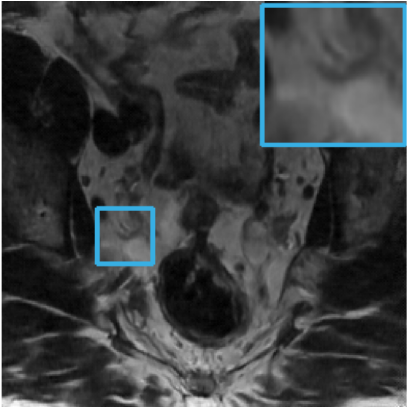}
    \caption{U-MRI}
    \end{subfigure}
    \hfill
    \begin{subfigure}[b]{0.16\textwidth}
    \includegraphics[width=\linewidth]{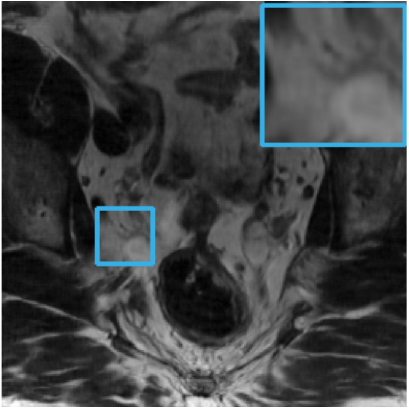}
    \caption{HUMUS}
    \end{subfigure}
    \hfill
    \begin{subfigure}[b]{0.16\textwidth}
    \includegraphics[width=\linewidth]{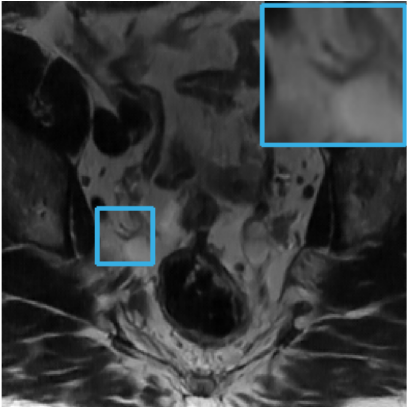}
    \caption{U-LDA}
    \end{subfigure}
    \hfill
    \begin{subfigure}[b]{0.16\textwidth}
    \includegraphics[width=\linewidth]{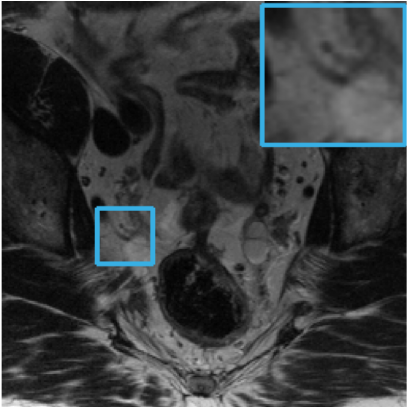}
    \caption{True}
    \end{subfigure}
    \caption{Comparison of cross-anatomy TL reconstruction results using 20\% Cartesian sampling ratio in Section \ref{subsubsec:cross-anatomy}. Top two rows show two instances of reconstructed cardiac images. Bottom rows show two instances of reconstructed prostate images. Blue squares zoom in the corresponding small squares in the images.}
    \label{fig:cross-anatomy}
\end{figure}

\subsubsection{Cross-sampling-ratio transfer}
\label{subsubsec:cross-sampling}

We now evaluate our method in TL across different sampling ratios.
Sampling ratios often vary substantially in different acquisition patterns. 
It is difficult to obtain sufficient training data for the reconstruction task with a new specific sampling ratio. In this test, we apply our method to learn knowledge from large data with various sampling ratios and transfer it to reconstruct images where data is at new unseen sampling ratios.

We have 300 $k$-space data obtained from the randomly selected brain images in the fastMRI dataset. For each of the sampling ratios 10\%, 20\%, and 30\%, we form $D_1$, $D_2$, and $D_3$. We train the feature-extractor $\gbf$ on them as in Step 1 and then the adapters $\hat{\hbf}_j$'s on small $k$-space data $\hat{D}_1$ and $\hat{D}_2$ at new sampling ratios at 15\% and 25\% respectively in Step 2. Both $\hat{D}_1$ and $\hat{D}_2$ have 100 images. 
The adapters $\hbf_i$'s trained in Step 1 and Step 2 are also used to test the performance of $\hbf_i \circ \gbf$ on the 200 randomly selected testing objects/images for sampling ratios 10\%, 20\%, and 30\%. 
In testing, we randomly sample 200 images from each of the datasets with 10\%, 20\%, 30\%, 15\% and 25\% sampling ratios (excluding those used to train $\gbf, \hbf_i ,\hat{\hbf}_j$'s), and apply $\hbf_i \circ \gbf$ and $\hat{\hbf}_j \circ \gbf$ to the corresponding datasets to check the reconstruction quality.

Similar to Table \ref{tab:MRI-anatomy}, we show the average PSNR and SSIM of reconstructed images with different sampling ratios in Table~\ref{tab:MRI-sampling}. 
We again observe that the proposed U-LDA method demonstrates significant improvements in reconstruction quality by transferring knowledge learned from reconstructions on sampling ratios of 10\%, 20\%, and 30\% to the cases of 15\% and 25\%. The proposed U-LDA continuously outperforms other non-TL and TL methods, except for the 30\% case where the LDA method performs slightly better even without TL because of the enriched $k$-space measurements. 
We also randomly select one image from each of the five sampling ratios and show the images obtained by different compared methods in Figure~\ref{fig:cross-sampling}, which demonstrates improved reconstruction quality using the proposed U-LDA. 
\begin{table}[t]
    \centering
    \caption{Average PSNR and SSIM of reconstructed images with sampling ratios 10\%, 15\%, 20\%, 25\%, and 30\% in Section \ref{subsubsec:cross-sampling}. The sampling ratios marked by $\dagger$ (15\% and 25\%) have small training datasets. For each sampling ratio, the same Cartesian mask is used to undersample the $k$-spaces of all images.}
    \scalebox{0.9}{
    \begin{tabular}{ccccccccccc}
    \toprule
    \multirow{3}{*}{Method} & \multicolumn{5}{c}{PSNR (dB)} & \multicolumn{5}{c}{SSIM (\%)}\\
    \cmidrule(l{1em}r{1em}){2-6} \cmidrule(l{1em}r{1em}){7-11}
    & $10\%$ & $20\%$ & $30\%$ & $15\%^\dagger$ & $25\%^\dagger$ & $10\%$ & $20\%$ & $30\%$ & $15\%^\dagger$ & $25\%^\dagger$\\
    \addlinespace
    \cmidrule{1-11}
    DnCn & 29.85 & 35.85 & 39.29 & 31.93 & 35.75 & 86.82 & 94.37 & 96.55 & 89.92 & 94.18\\
    LDA & 31.27 & 37.11 & \textbf{40.67} & 33.26 & 37.80 & 89.36 & 95.36 & \textbf{97.17} & 92.13 & 95.83\\
    UNet & 26.93 & 29.16 & 31.20 & 28.03 & 29.97 & 76.17 & 79.55 & 83.49 & 77.63 & 81.11\\
    U-MRI & 30.60 & 36.25 & 39.42 & 33.58 & 37.33 & 86.07 & 94.54 & 96.49 & 91.91 & 95.22\\
    Meta & 27.63 & 32.97 & 35.73 & 30.72 & 34.44 & 82.36 & 91.22 & 94.06 & 87.86 & 92.87\\
    HUMUS-Net & 28.85 & 32.21 & 34.95 & 30.17 & 33.21 & 85.55 & 90.92 & 94.00 & 87.77 & 92.26 \\
    U-LDA (Ours) & \textbf{32.13} & \textbf{37.54} & 40.48 & \textbf{34.90} & \textbf{38.80} & \textbf{90.72} & \textbf{95.62} & 97.13 & \textbf{93.70} & \textbf{96.37}\\
    \bottomrule
    \end{tabular}}
    \label{tab:MRI-sampling}
\end{table}
\begin{figure}[h!]
    \centering
    \begin{subfigure}[b]{0.16\textwidth}
        \includegraphics[width=\linewidth]{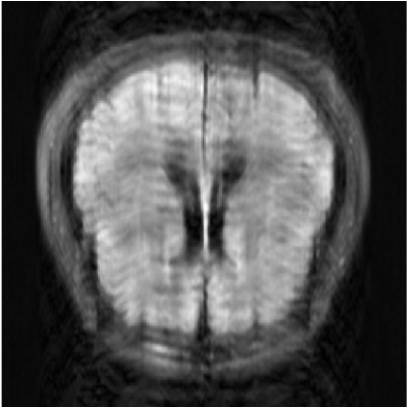}
    \end{subfigure}
    \hfill
    \begin{subfigure}[b]{0.16\textwidth}
        \includegraphics[width=\linewidth]{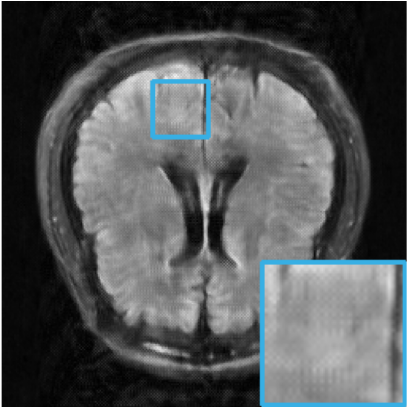}
    \end{subfigure}
    \hfill
    \begin{subfigure}[b]{0.16\textwidth}
        \includegraphics[width=\linewidth]{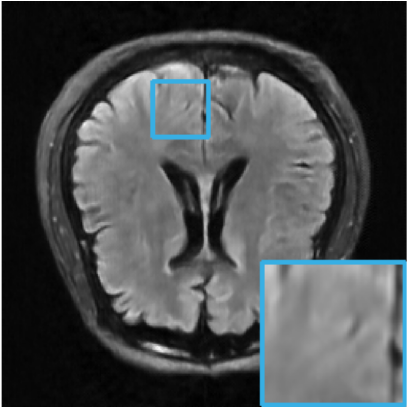}
    \end{subfigure}
    \hfill
    \begin{subfigure}[b]{0.16\textwidth}
        \includegraphics[width=\linewidth]{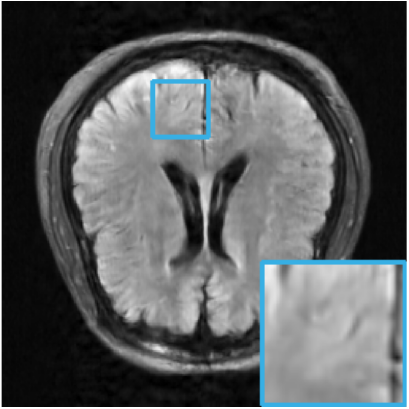}
    \end{subfigure}
    \hfill
    \begin{subfigure}[b]{0.16\textwidth}
        \includegraphics[width=\linewidth]{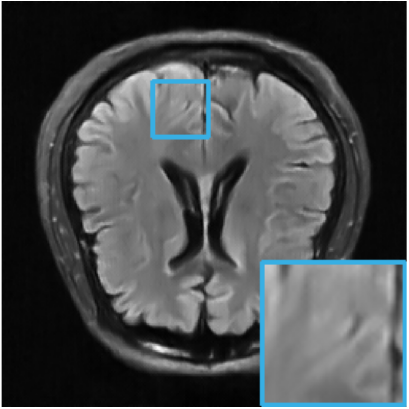}
    \end{subfigure}
    \hfill
    \begin{subfigure}[b]{0.16\textwidth}
        \includegraphics[width=\linewidth]{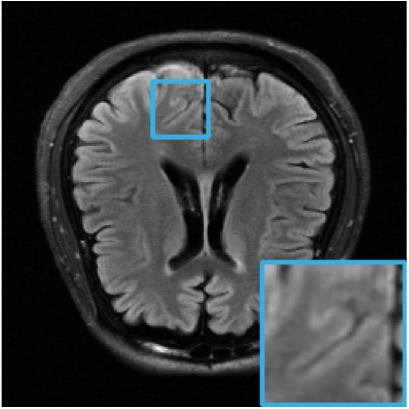}
    \end{subfigure}
    \vspace{3pt}
    \begin{subfigure}[b]{0.16\textwidth}
        \includegraphics[width=\linewidth]{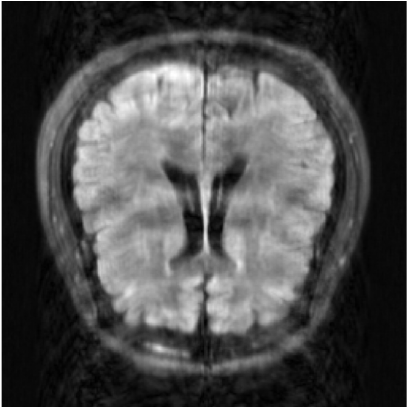}
    \end{subfigure}
    \hfill
    \begin{subfigure}[b]{0.16\textwidth}
        \includegraphics[width=\linewidth]{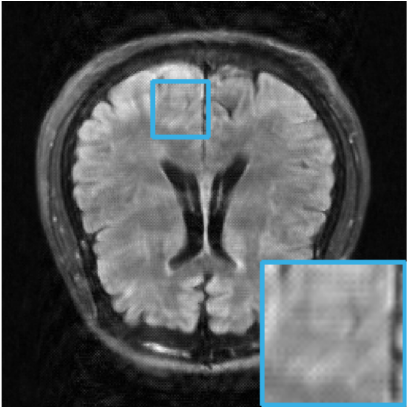}
    \end{subfigure}
    \hfill
    \begin{subfigure}[b]{0.16\textwidth}
        \includegraphics[width=\linewidth]{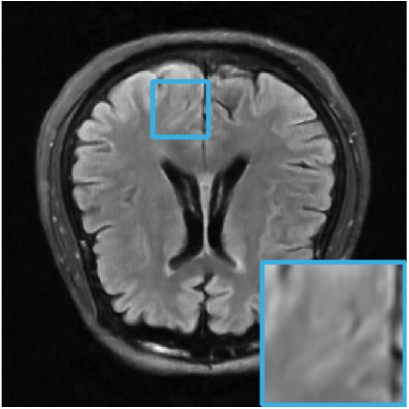}
    \end{subfigure}
    \hfill
    \begin{subfigure}[b]{0.16\textwidth}
        \includegraphics[width=\linewidth]{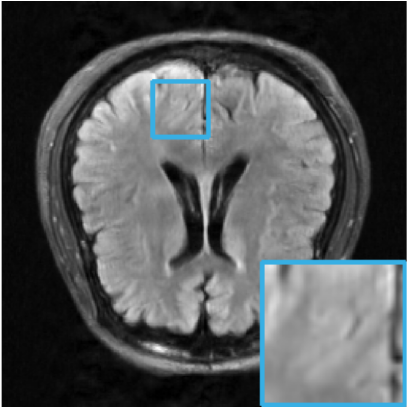}
    \end{subfigure}
    \hfill
    \begin{subfigure}[b]{0.16\textwidth}
        \includegraphics[width=\linewidth]{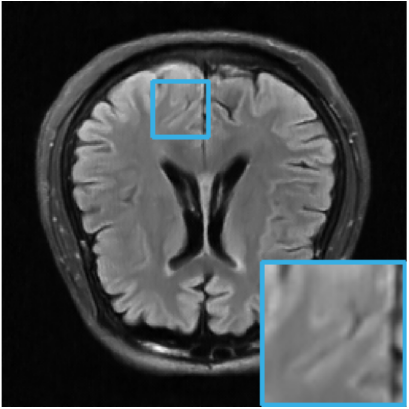}
    \end{subfigure}
    \hfill
    \begin{subfigure}[b]{0.16\textwidth}
        \includegraphics[width=\linewidth]{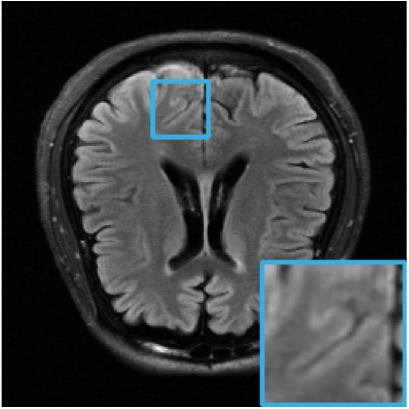}
    \end{subfigure}
    \begin{subfigure}[b]{0.16\textwidth}
        \includegraphics[width=\linewidth]{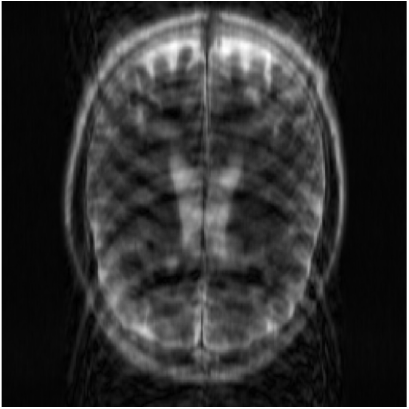}
    \end{subfigure}
    \hfill
    \begin{subfigure}[b]{0.16\textwidth}
        \includegraphics[width=\linewidth]{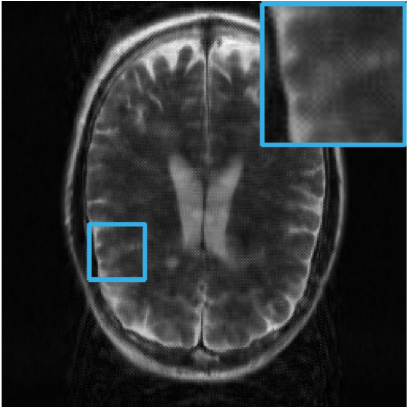}
    \end{subfigure}
    \hfill
    \begin{subfigure}[b]{0.16\textwidth}
        \includegraphics[width=\linewidth]{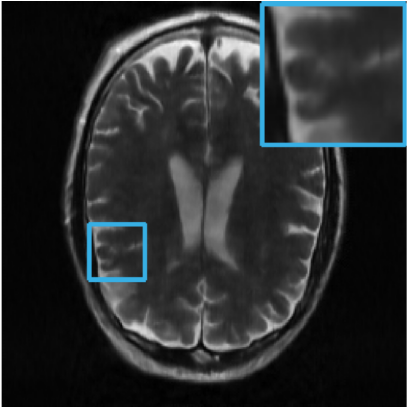}
    \end{subfigure}
    \hfill
    \begin{subfigure}[b]{0.16\textwidth}
        \includegraphics[width=\linewidth]{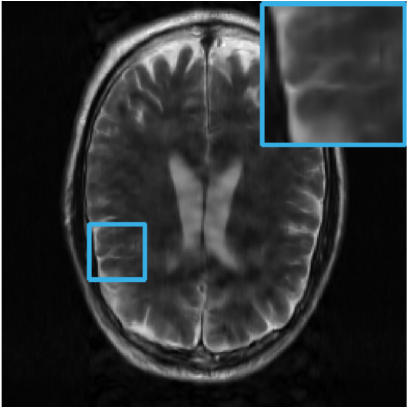}
    \end{subfigure}
    \hfill
    \begin{subfigure}[b]{0.16\textwidth}
        \includegraphics[width=\linewidth]{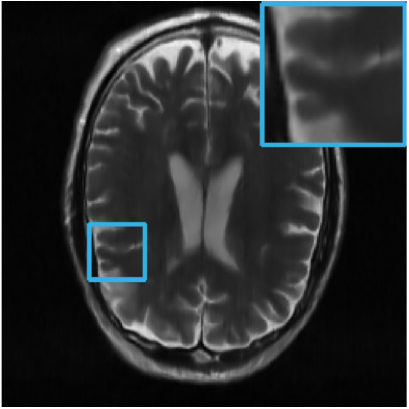}
    \end{subfigure}
    \hfill
    \begin{subfigure}[b]{0.16\textwidth}
        \includegraphics[width=\linewidth]{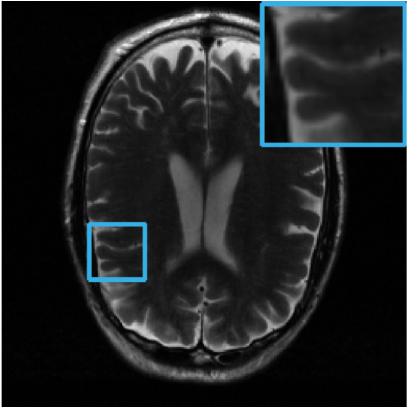}
    \end{subfigure}
    \begin{subfigure}[b]{0.16\textwidth}
        \includegraphics[width=\linewidth]{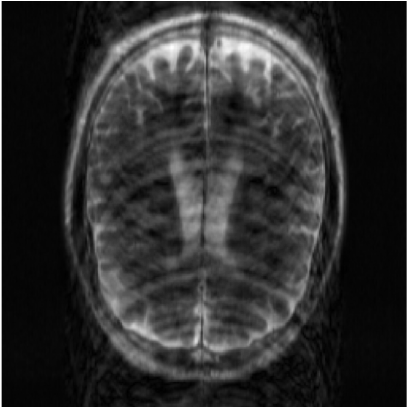}
        \caption{Zero-filling}
    \end{subfigure}
    \hfill
    \begin{subfigure}[b]{0.16\textwidth}
        \includegraphics[width=\linewidth]{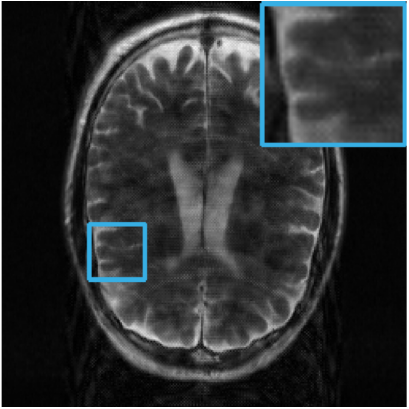}
        \caption{UNet}
    \end{subfigure}
    \hfill
    \begin{subfigure}[b]{0.16\textwidth}
        \includegraphics[width=\linewidth]{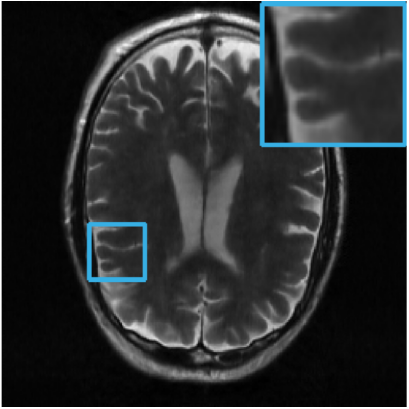}
        \caption{U-MRI}
    \end{subfigure}
    \hfill
    \begin{subfigure}[b]{0.16\textwidth}
        \includegraphics[width=\linewidth]{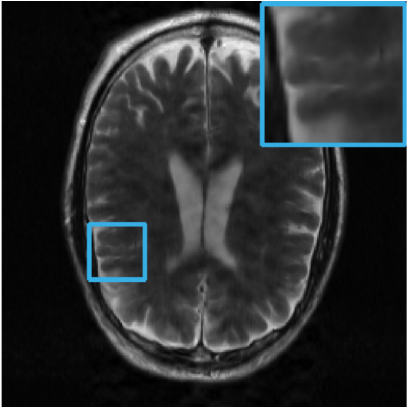}
        \caption{HUMUS}
    \end{subfigure}
    \hfill
    \begin{subfigure}[b]{0.16\textwidth}
        \includegraphics[width=\linewidth]{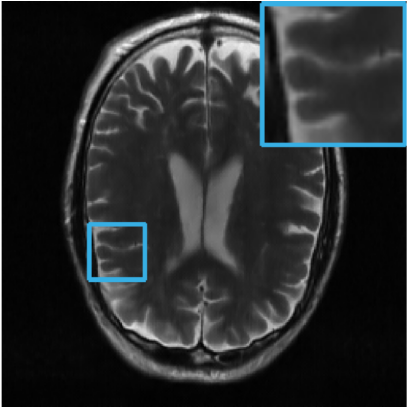}
        \caption{U-LDA}
    \end{subfigure}
    \hfill
    \begin{subfigure}[b]{0.16\textwidth}
        \includegraphics[width=\linewidth]{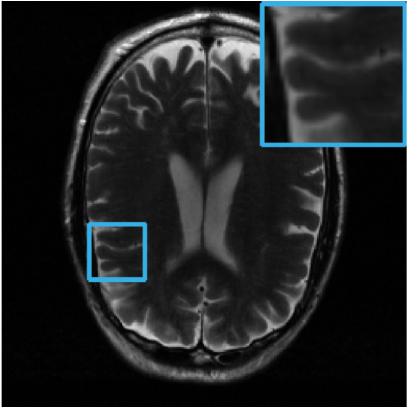}
        \caption{True}
    \end{subfigure}
    \caption{Comparison of cross-sampling-ratio TL reconstruction results in Section \ref{subsubsec:cross-sampling}. Top two rows show two instances of reconstructed images under 15\% sampling ratio. Bottom rows show two instances of reconstructed images under 25\% sampling ratio. Blue squares zoom in the corresponding small squares in the images.}
    \label{fig:cross-sampling}
\end{figure}

\subsubsection{Cross-modality transfer}
\label{subsubsec:cross-modality}

In this experiment, we test a cross-modality TL problem by learning information from natural images and adapting to MRI reconstruction using the proposed two-step framework.
We use the ImageNet and CIFAR-10 data sets as the source domain $D=\{D_1, D_2\}$. Specifically, we randomly select 400 grayscale images from each of the ImageNet and the CIFAR-10 data sets, normalize their pixel intensities to $[0, 1]$, and apply the Fourier transform to obtain their corresponding $k$-space data. 
Then we apply a $k$-space cartesian mask with 20\% sampling ratio to these data. 
We train $\gbf$ and $\hbf_i$'s on $D$ of natural images in Step 1, and use it to train $\hat{\hbf}_j$'s in Step 2 using the two MRI datasets fastMRI and Stanford2D, where 100 images are randomly selected and they form $\hat{D}_1$ and $\hat{D}_2$. We evaluate the performance on the test sets of both the natural images and MR images. To test the performance of all the compared methods, we randomly sample 200 images from each of the ImageNet, CIFAR-10, fastMRI and Stanford2D data sets, excluding those already in $\{D_1, D_2, \hat{D}_1, \hat{D}_2\}$.

We show the average PSNR and SSIM of images reconstructed by all compared methods in Table \ref{tab:nature-to-mri}. 
In Table \ref{tab:nature-to-mri}, we see that the proposed U-LDA method demonstrates a strong ability of transferring knowledge learned from natural images to medical images, which are of completely different modalities. 
We also randomly select one image from each of the four datasets and show the images obtained by different compared methods in Figure~\ref{fig:cross-modality}, which again suggests improved reconstruction quality using the proposed U-LDA. 
\begin{table}[t]
    \centering
    \caption{Average PSNR and SSIM of reconstructed images on different modality domains in Section \ref{subsubsec:cross-modality}. The modalities marked by $\dagger$ (fastMRI and Stanford2D) have small training datasets.}
    \scalebox{0.75}{
    \begin{tabular}{ccccccccc}
    \toprule
    \multirow{3}{*}{Method} & \multicolumn{4}{c}{PSNR (dB)} & \multicolumn{4}{c}{SSIM (\%)}\\
    \cmidrule(l{1em}r{1em}){2-5} \cmidrule(l{1em}r{1em}){6-9}
    & ImageNet & CIFAR-10 &$\text{fastMRI}^\dagger$ & $\text{Stanford2D}^\dagger$ & ImageNet & CIFAR-10&$\text{fastMRI}^\dagger$ & $\text{Stanford2D}^\dagger$\\
    \addlinespace
    \cmidrule{1-9}
    DnCn & 33.90 & 38.68 & 31.90 &  36.29 & 92.87 & 97.09 & 84.97 & 92.33\\
    LDA & 35.37 & 41.33 & 33.42 & 37.85 & 94.26 & 97.99 & 87.30 & 94.23 \\
    UNet & 31.36 & 36.27 & 30.88 & 35.33 & 89.94 & 96.40 & 82.58 & 91.00 \\
    U-MRI & 34.33 & 39.72 & 30.50 & 34.48 & 93.37 & 97.66 & 80.96 & 89.08 \\
    Meta & 32.81 & 38.11 & 30.93 & 35.04 & 90.95 & 96. 26 & 82.35 & 90.47 \\
    HUMUS-Net & 34.52 & 42.25 & 32.38 & 36.94 & 93.97 & 98.73 & 85.49 & 93.04 \\
    U-LDA & \textbf{36.66} & \textbf{43.21} & \textbf{33.58} & \textbf{38.17} & \textbf{95.56} & \textbf{98.75} & \textbf{87.34} & \textbf{94.41} \\
    \bottomrule
    \end{tabular}}
    \label{tab:nature-to-mri}
\end{table}
\begin{figure}[h!]
    \centering
    \begin{subfigure}[b]{0.16\textwidth}
        \includegraphics[width=\linewidth]{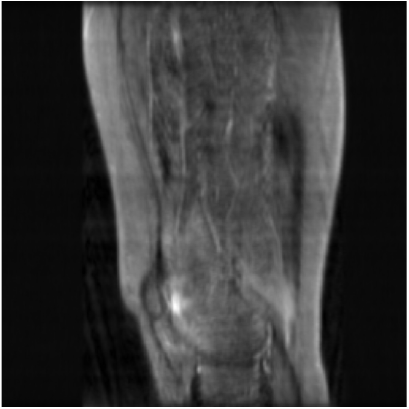}
    \end{subfigure}
    \hfill
    \begin{subfigure}[b]{0.16\textwidth}
        \includegraphics[width=\linewidth]{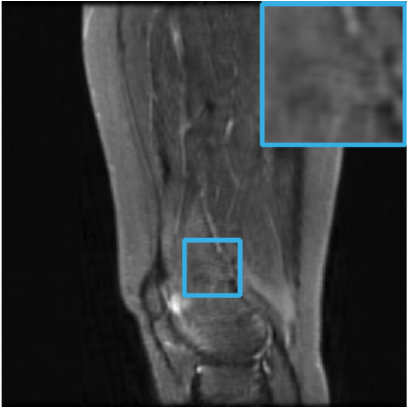}
    \end{subfigure}
    \hfill
    \begin{subfigure}[b]{0.16\textwidth}
        \includegraphics[width=\linewidth]{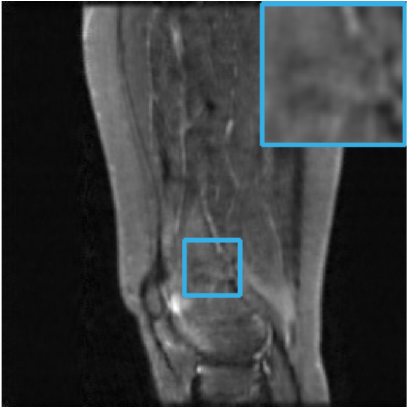}
    \end{subfigure}
    \hfill
    \begin{subfigure}[b]{0.16\textwidth}
        \includegraphics[width=\linewidth]{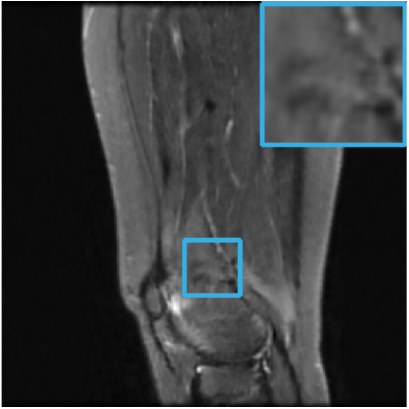}
    \end{subfigure}
    \hfill
    \begin{subfigure}[b]{0.16\textwidth}
        \includegraphics[width=\linewidth]{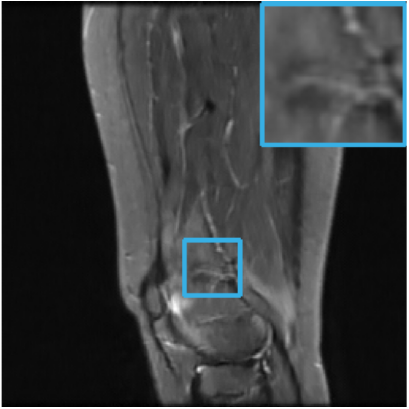}
    \end{subfigure}
    \hfill
    \begin{subfigure}[b]{0.16\textwidth}
        \includegraphics[width=\linewidth]{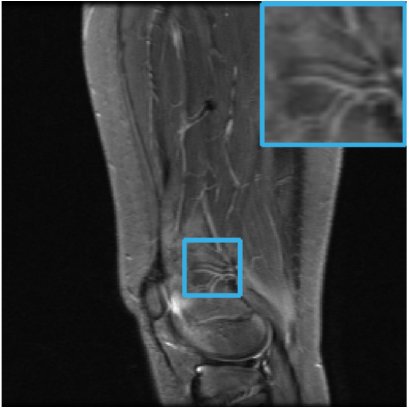}
    \end{subfigure}
    \begin{subfigure}[b]{0.16\textwidth}
        \includegraphics[width=\linewidth]{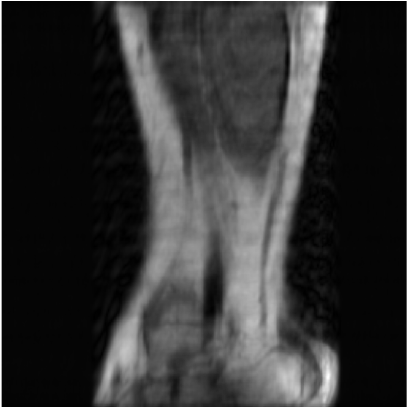}
    \end{subfigure}
    \hfill
    \begin{subfigure}[b]{0.16\textwidth}
        \includegraphics[width=\linewidth]{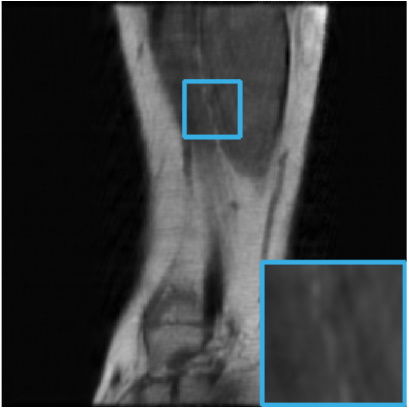}
    \end{subfigure}
    \hfill
    \begin{subfigure}[b]{0.16\textwidth}
        \includegraphics[width=\linewidth]{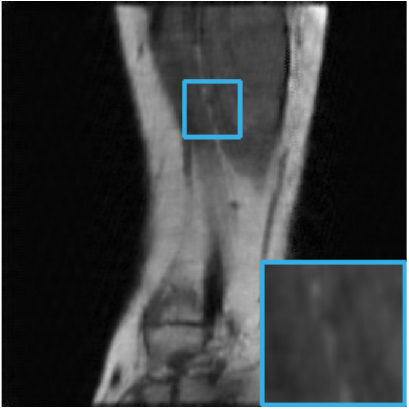}
    \end{subfigure}
    \hfill
    \begin{subfigure}[b]{0.16\textwidth}
        \includegraphics[width=\linewidth]{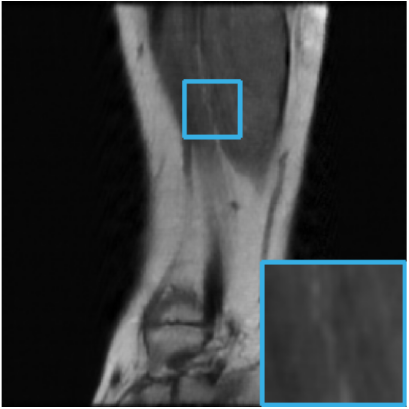}
    \end{subfigure}
    \hfill
    \begin{subfigure}[b]{0.16\textwidth}
        \includegraphics[width=\linewidth]{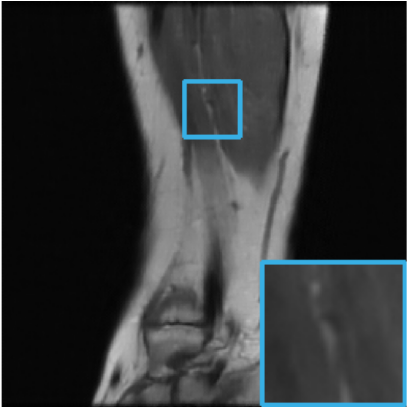}
    \end{subfigure}
    \hfill
    \begin{subfigure}[b]{0.16\textwidth}
        \includegraphics[width=\linewidth]{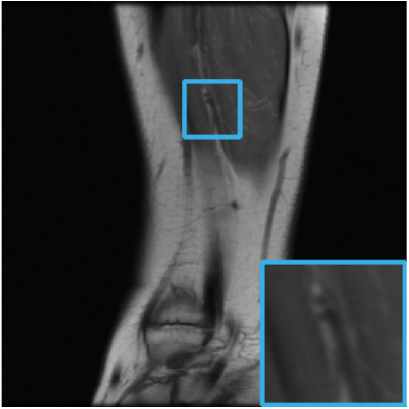}
    \end{subfigure}
    \begin{subfigure}[b]{0.16\textwidth}
        \includegraphics[width=\linewidth]{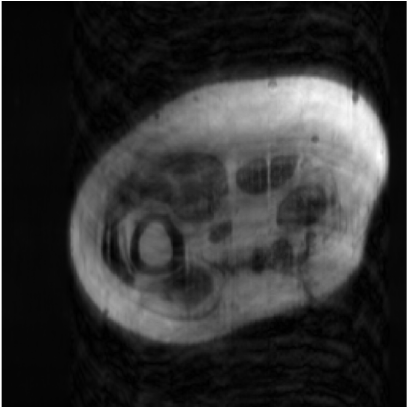}
    \end{subfigure}
    \hfill
    \begin{subfigure}[b]{0.16\textwidth}
        \includegraphics[width=\linewidth]{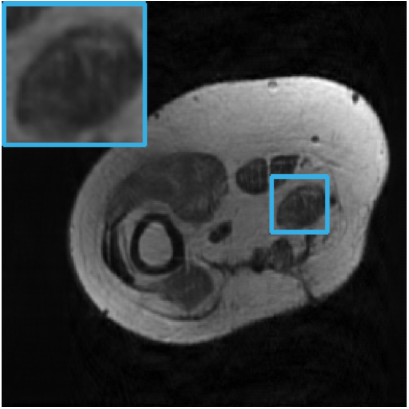}
    \end{subfigure}
    \hfill
    \begin{subfigure}[b]{0.16\textwidth}
        \includegraphics[width=\linewidth]{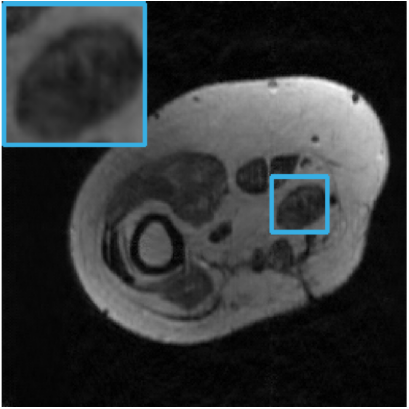}

    \end{subfigure}
    \hfill
    \begin{subfigure}[b]{0.16\textwidth}
        \includegraphics[width=\linewidth]{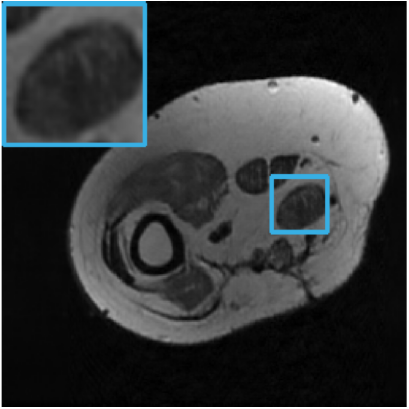}
    \end{subfigure}
    \hfill
    \begin{subfigure}[b]{0.16\textwidth}
        \includegraphics[width=\linewidth]{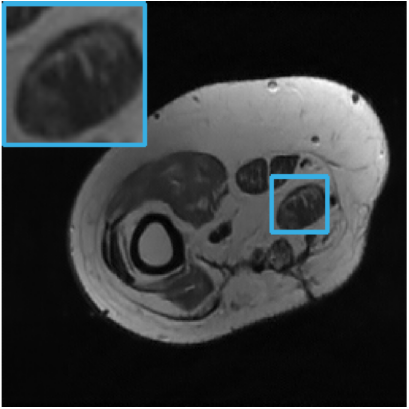}
    \end{subfigure}
    \hfill
    \begin{subfigure}[b]{0.16\textwidth}
        \includegraphics[width=\linewidth]{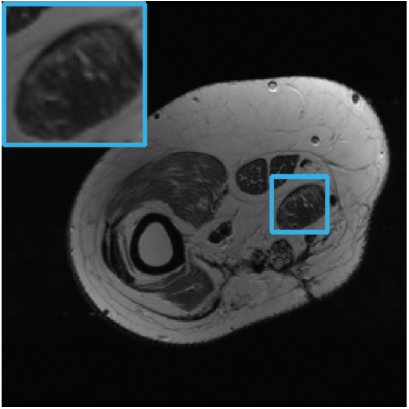}
    \end{subfigure}
    \begin{subfigure}[b]{0.16\textwidth}
        \includegraphics[width=\linewidth]{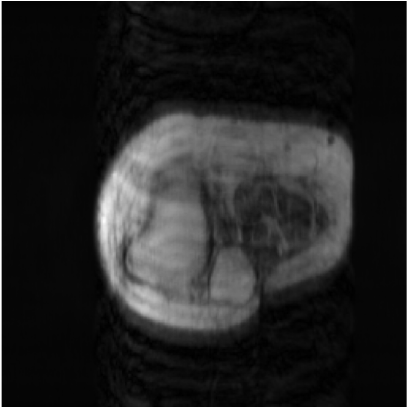}
        \caption{Zero-filling}
    \end{subfigure}
    \hfill
    \begin{subfigure}[b]{0.16\textwidth}
        \includegraphics[width=\linewidth]{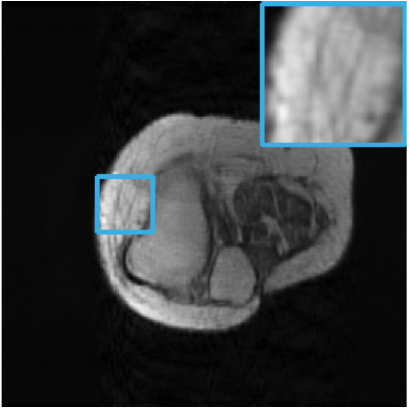}
        \caption{UNet}
    \end{subfigure}
    \hfill
    \begin{subfigure}[b]{0.16\textwidth}
        \includegraphics[width=\linewidth]{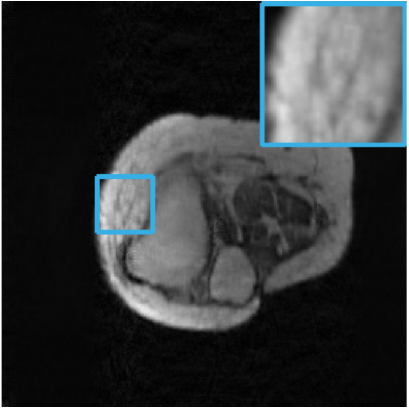}
        \caption{U-MRI}
    \end{subfigure}
    \hfill
    \begin{subfigure}[b]{0.16\textwidth}
        \includegraphics[width=\linewidth]{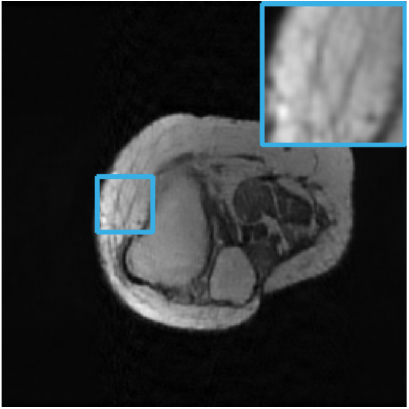}
        \caption{HUMUS}
    \end{subfigure}
    \hfill
    \begin{subfigure}[b]{0.16\textwidth}
        \includegraphics[width=\linewidth]{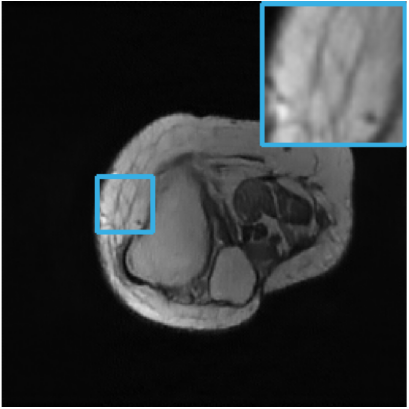}
        \caption{U-LDA}
    \end{subfigure}
    \hfill
    \begin{subfigure}[b]{0.16\textwidth}
        \includegraphics[width=\linewidth]{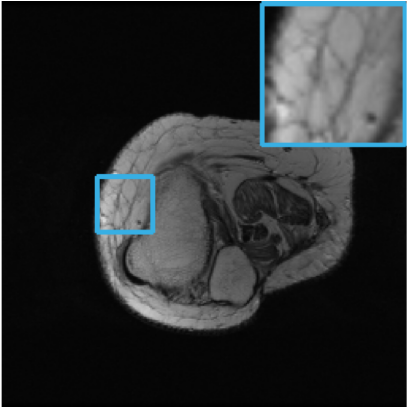}
        \caption{True}
    \end{subfigure}
    \caption{Comparison of cross-modality reconstruction results in Section \ref{subsubsec:cross-modality}. Knowledge is transferred from natural image datasets ImageNet and CIFAR-10 to medical image reconstruction. Top two rows show two instances of reconstructed images in the fastMRI dataset. Bottom rows show two instances of reconstructed images in the Stanford2D dataset. Blue squares zoom in the corresponding small squares in the images.}
    \label{fig:cross-modality}
\end{figure}

\subsection{Ablation study and discussions}
\label{subsec:ablation}

In this section, we conduct several ablation studies to further investigate the performance of the proposed method.

\paragraph{Pointwise absolute reconstruction error}
{We check the pointwise absolute errors of the reconstructed images in Sections \ref{subsubsec:cross-anatomy}, \ref{subsubsec:cross-sampling}, and \ref{subsubsec:cross-modality}.
In addition, we further zoom in the local region of the reconstructed images and see the reconstruction details. 
Specifically, we select one image from each of the experiments presented in Sections \ref{subsubsec:cross-anatomy}, \ref{subsubsec:cross-sampling}, and \ref{subsubsec:cross-modality}, corresponding to cross-anatomy, cross-sampling-ratio, and cross-modality reconstructions, respectively. 
We plot the zoom-in views and pointwise absolute errors of these images from their corresponding ground truth images Figure \ref{fig:zoom-absdiff-anatomy}. 
These plots indicate that the proposed U-LDA is capable of recovering fine structures without introducing notorious artifacts.
}
\begin{figure}[h!]
    \centering
    %
    \begin{subfigure}[b]{0.14\textwidth}
        \includegraphics[width=\linewidth]{figs/C_A_card/UNet_cardiac_1.pdf}
    \end{subfigure}
    \begin{subfigure}[b]{0.14\textwidth}
        \includegraphics[width=\linewidth]{figs/C_A_card/UMRI_cardiac_1.pdf}
    \end{subfigure}
    \begin{subfigure}[b]{0.14\textwidth}
        \includegraphics[width=\linewidth]{figs/C_A_card/HUMUS_cardiac_1.pdf}
    \end{subfigure}
    \begin{subfigure}[b]{0.14\textwidth}
        \includegraphics[width=\linewidth]{figs/C_A_card/ULDA_cardiac_1.pdf}
    \end{subfigure}
    \begin{subfigure}[b]{0.14\textwidth}
        \includegraphics[width=\linewidth]{figs/C_A_card/GT_cardiac_1.pdf}
    \end{subfigure}
    \\
    \begin{subfigure}[b]{0.14\textwidth}
        \includegraphics[width=\linewidth]{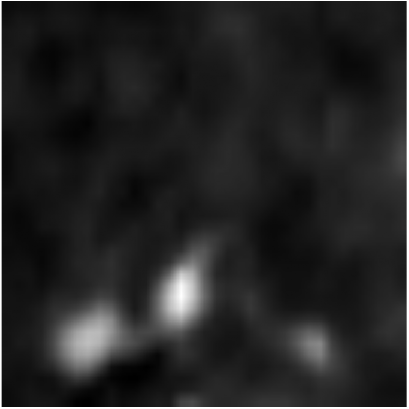}
    \end{subfigure}
    \begin{subfigure}[b]{0.14\textwidth}
        \includegraphics[width=\linewidth]{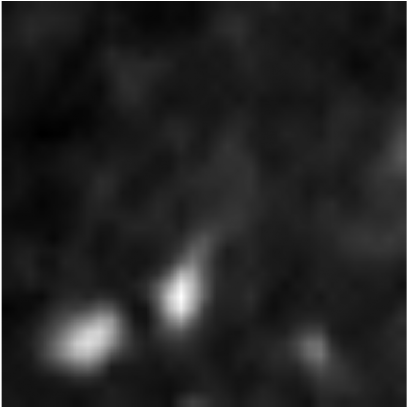}
    \end{subfigure}
    \begin{subfigure}[b]{0.14\textwidth}
        \includegraphics[width=\linewidth]{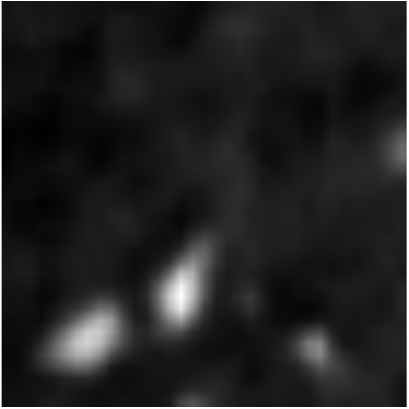}
    \end{subfigure}
    \begin{subfigure}[b]{0.14\textwidth}
        \includegraphics[width=\linewidth]{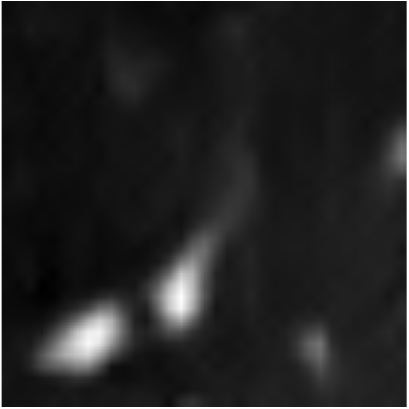}
    \end{subfigure}
    \begin{subfigure}[b]{0.14\textwidth}
        \includegraphics[width=\linewidth]{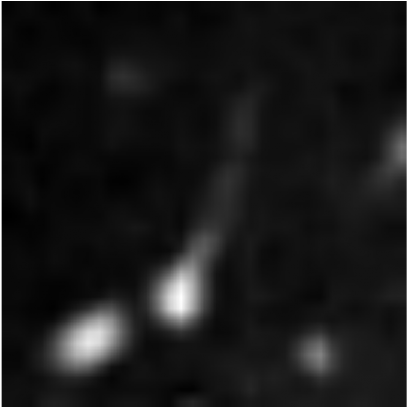}
    \end{subfigure}
    \\
    \begin{subfigure}[b]{0.14\textwidth}
        \includegraphics[width=\linewidth]{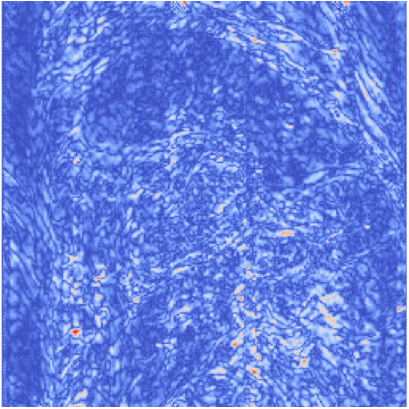}
    \end{subfigure}
    \begin{subfigure}[b]{0.14\textwidth}
        \includegraphics[width=\linewidth]{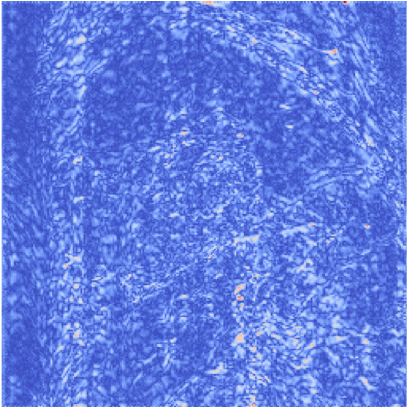}
    \end{subfigure}
    \begin{subfigure}[b]{0.14\textwidth}
        \includegraphics[width=\linewidth]{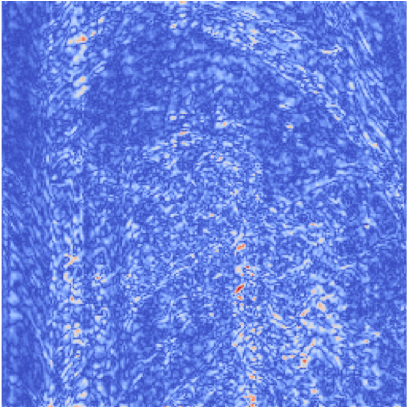}
    \end{subfigure}
    \begin{subfigure}[b]{0.14\textwidth}
        \includegraphics[width=\linewidth]{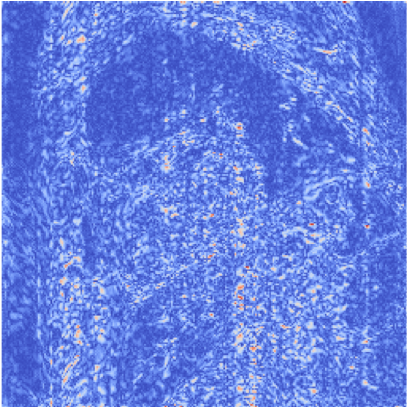}
    \end{subfigure}
    \begin{subfigure}[b]{0.14\textwidth}
        \phantom{\includegraphics[width=\linewidth]{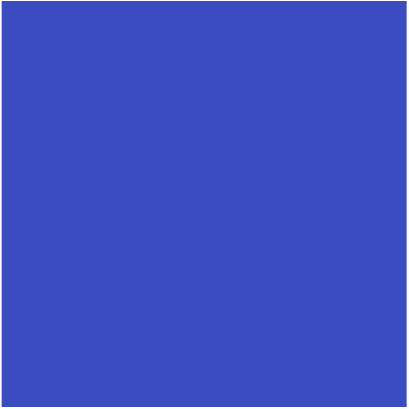}}
    \end{subfigure}
    \\
    \begin{subfigure}[b]{0.14\textwidth}
        \includegraphics[width=\linewidth]{figs/C_S_brain_6/UNet_brain_6.66_0.pdf}
    \end{subfigure}
    \begin{subfigure}[b]{0.14\textwidth}
        \includegraphics[width=\linewidth]{figs/C_S_brain_6/UMRI_brain_6.66_0.pdf}
    \end{subfigure}
    \begin{subfigure}[b]{0.14\textwidth}
        \includegraphics[width=\linewidth]{figs/C_S_brain_6/HUMUS_brain_6.66_0.pdf}
    \end{subfigure}
    \begin{subfigure}[b]{0.14\textwidth}
        \includegraphics[width=\linewidth]{figs/C_S_brain_6/ULDA_brain_6.66_0.pdf}
    \end{subfigure}
    \begin{subfigure}[b]{0.14\textwidth}
        \includegraphics[width=\linewidth]{figs/C_S_brain_6/GT_brain_6.66_0.pdf}
    \end{subfigure}
    \\
    \begin{subfigure}[b]{0.14\textwidth}
        \includegraphics[width=\linewidth]{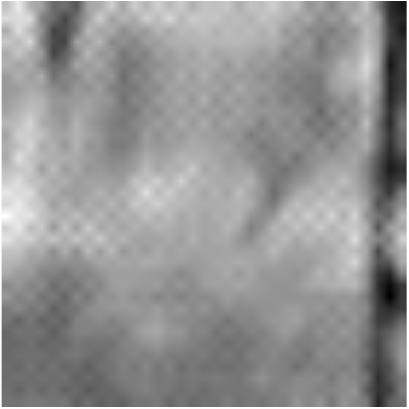}
    \end{subfigure}
    \begin{subfigure}[b]{0.14\textwidth}
        \includegraphics[width=\linewidth]{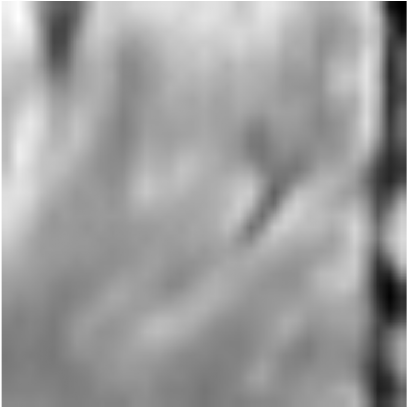}
    \end{subfigure}
    \begin{subfigure}[b]{0.14\textwidth}
        \includegraphics[width=\linewidth]{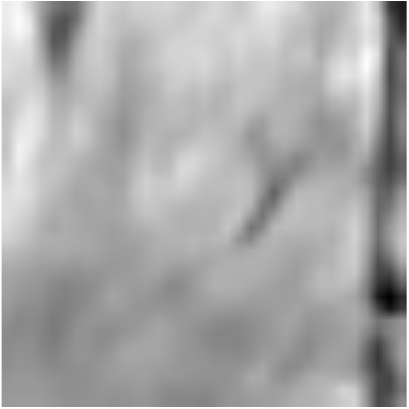}
    \end{subfigure}
    \begin{subfigure}[b]{0.14\textwidth}
        \includegraphics[width=\linewidth]{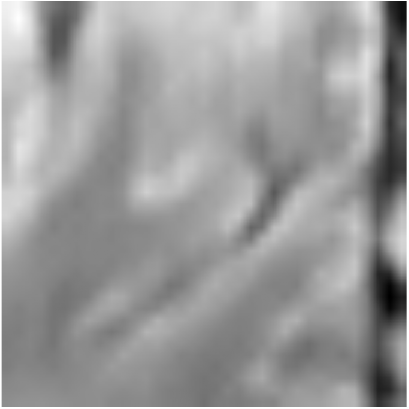}
    \end{subfigure}
    \begin{subfigure}[b]{0.14\textwidth}
        \includegraphics[width=\linewidth]{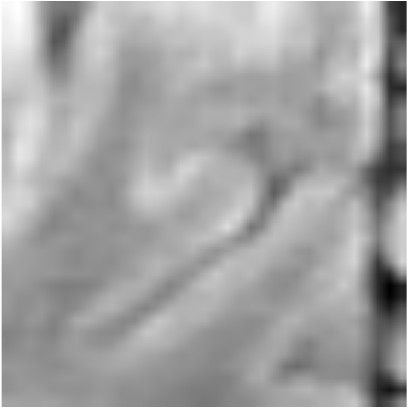}
    \end{subfigure}
    \\
    \begin{subfigure}[b]{0.14\textwidth}
        \includegraphics[width=\linewidth]{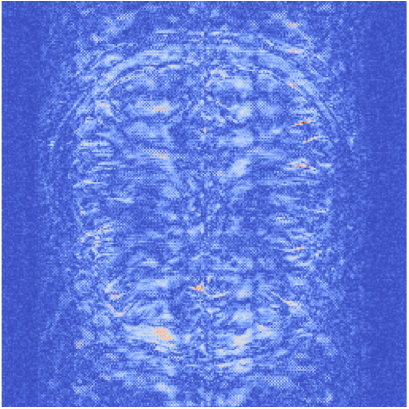}
    \end{subfigure}
    \begin{subfigure}[b]{0.14\textwidth}
        \includegraphics[width=\linewidth]{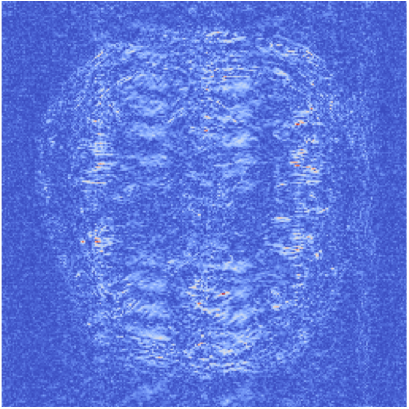}
    \end{subfigure}
    \begin{subfigure}[b]{0.14\textwidth}
        \includegraphics[width=\linewidth]{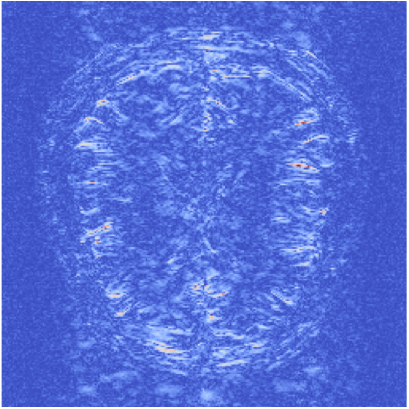}
    \end{subfigure}
    \begin{subfigure}[b]{0.14\textwidth}
        \includegraphics[width=\linewidth]{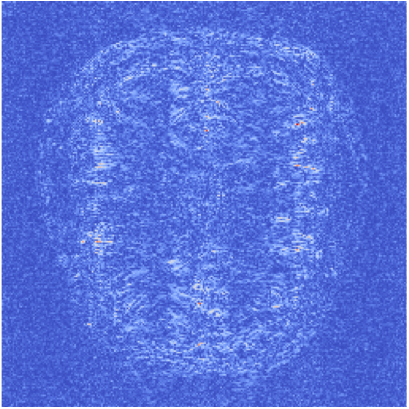}
    \end{subfigure}
    \begin{subfigure}[b]{0.14\textwidth}
        \phantom{\includegraphics[width=\linewidth]{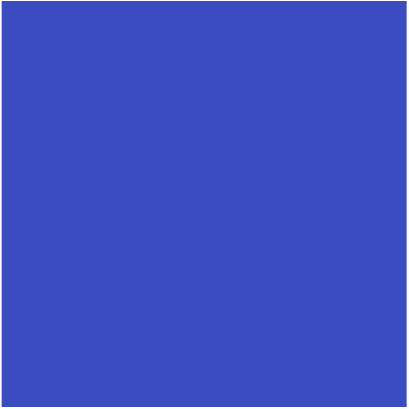}}
    \end{subfigure}
    \\
    \begin{subfigure}[b]{0.14\textwidth}
        \includegraphics[width=\linewidth]{figs/C_D_stanford/UNet_stanford_5_2.pdf}
    \end{subfigure}
    \begin{subfigure}[b]{0.14\textwidth}
        \includegraphics[width=\linewidth]{figs/C_D_stanford/UMRI_stanford_5_2.pdf}

    \end{subfigure}
    \begin{subfigure}[b]{0.14\textwidth}
        \includegraphics[width=\linewidth]{figs/C_D_stanford/HUMUS_stanford_5_2.pdf}
    \end{subfigure}
    \begin{subfigure}[b]{0.14\textwidth}
        \includegraphics[width=\linewidth]{figs/C_D_stanford/ULDA_stanford_5_2.pdf}
    \end{subfigure}
    \begin{subfigure}[b]{0.14\textwidth}
        \includegraphics[width=\linewidth]{figs/C_D_stanford/GT_stanford_5_2.pdf}
    \end{subfigure}
    \\
    \begin{subfigure}[b]{0.14\textwidth}
        \includegraphics[width=\linewidth]{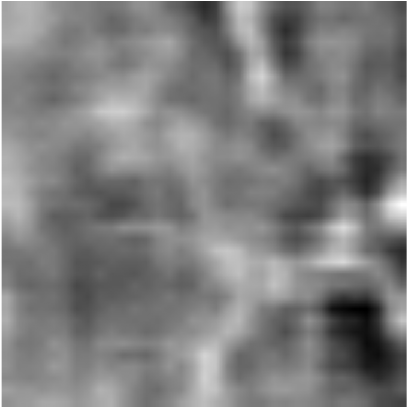}
    \end{subfigure}
    \begin{subfigure}[b]{0.14\textwidth}
        \includegraphics[width=\linewidth]{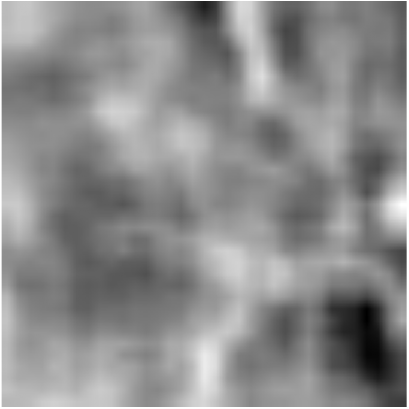}
    \end{subfigure}
    \begin{subfigure}[b]{0.14\textwidth}
        \includegraphics[width=\linewidth]{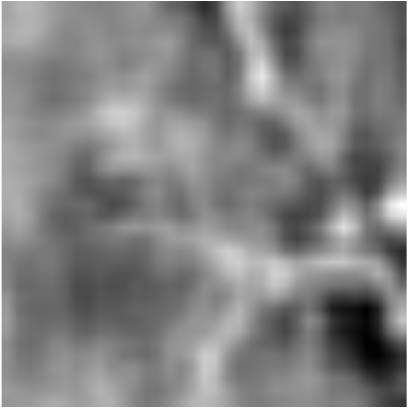}
    \end{subfigure}
    \begin{subfigure}[b]{0.14\textwidth}
        \includegraphics[width=\linewidth]{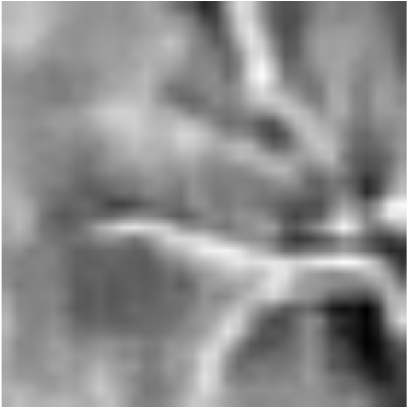}
    \end{subfigure}
    \begin{subfigure}[b]{0.14\textwidth}
        \includegraphics[width=\linewidth]{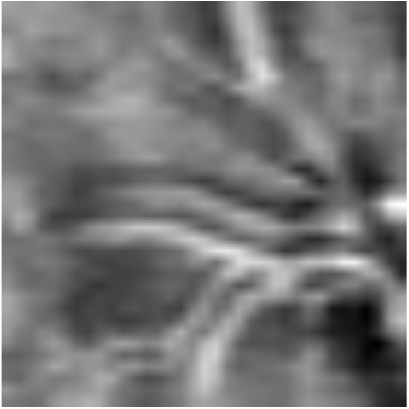}
    \end{subfigure}
    \\
    \begin{subfigure}[b]{0.14\textwidth}
        \includegraphics[width=\linewidth]{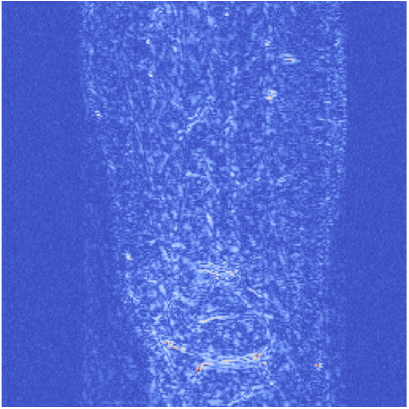}
        \caption{UNet}
    \end{subfigure}
    \begin{subfigure}[b]{0.14\textwidth}
        \includegraphics[width=\linewidth]{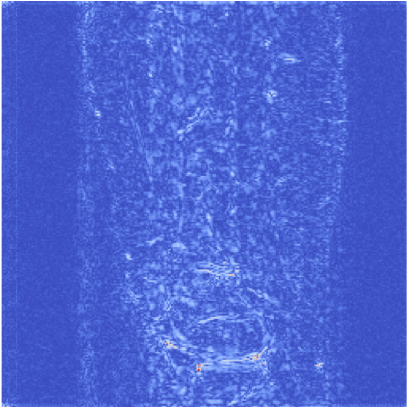}
        \caption{U-MRI}
    \end{subfigure}
    \begin{subfigure}[b]{0.14\textwidth}
        \includegraphics[width=\linewidth]{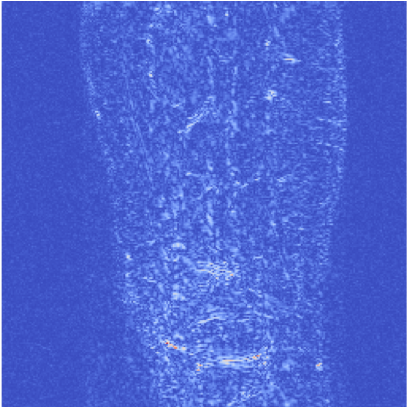}
        \caption{HUMUS}
    \end{subfigure}
    \begin{subfigure}[b]{0.14\textwidth}
        \includegraphics[width=\linewidth]{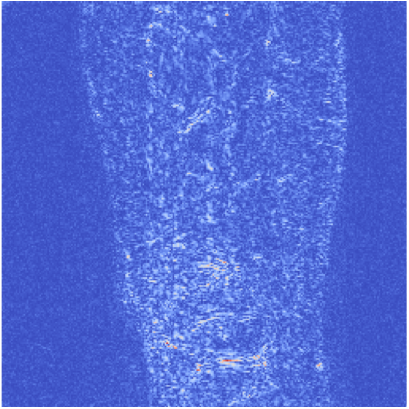}
        \caption{U-LDA}
    \end{subfigure}
    \begin{subfigure}[b]{0.14\textwidth}
        \phantom{\includegraphics[width=\linewidth]{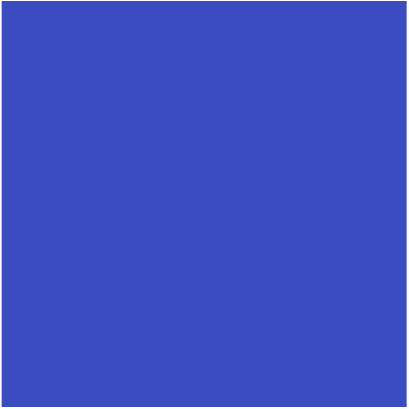}}
        \caption{True}
    \end{subfigure}
    \caption{Reconstructed images, zoom-in views, and pointwise absolute errors obtained by the compared methods on the cross-anatomy experiment (top three rows), cross-sampling-ratio experiment (middle three rows), and cross-modality experiment (bottom three rows).}
    \label{fig:zoom-absdiff-anatomy}
\end{figure}

\paragraph{Determining phase number}
To determine the phase number $T$ in U-LDA, we employ two statistical metrics: Normalized Cross-Correlation (NCC) \cite{sarvaiya2009image} and Central Kernel Alignment (CKA) \cite{kornblith2019similarity}. 
%
%
NCC measures the similarity between two signals, evaluating how well the features from the universal model correlate with those from independent models. CKA is a kernel-based method that quantifies the similarity between feature representations, capturing more complex relationships in high-dimensional spaces. 
Larger NCC and CKA imply stronger correlation and similarity.
%
%
To determine a proper $T$, we implement LDA which trains $\gbf_i$ for each $D_i$ and U-LDA with a universal feature-extractor $g$, and check the NCC and CKA between $\gbf_i$ and $\gbf$. 
When these metrics are large, we know that $\gbf_i$ and $\gbf$ have approximately reached agreement on the main image features. 
The results are shown in Figure \ref{fig:corr_anals}, where we plot the statistical metrics versus $T$. 
It appears that $T=15$ can serve as a proper phase number, which is the one we used in our other tests.
\begin{figure}[h!]
    \centering
    \begin{subfigure}[b]{\textwidth}
        \includegraphics[width=\linewidth]{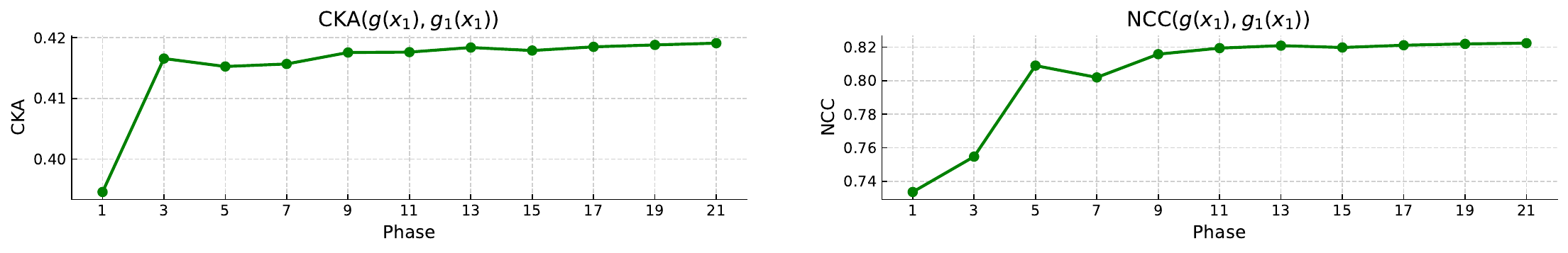}
    \end{subfigure}
    \begin{subfigure}[b]{\textwidth}
        \includegraphics[width=\linewidth]{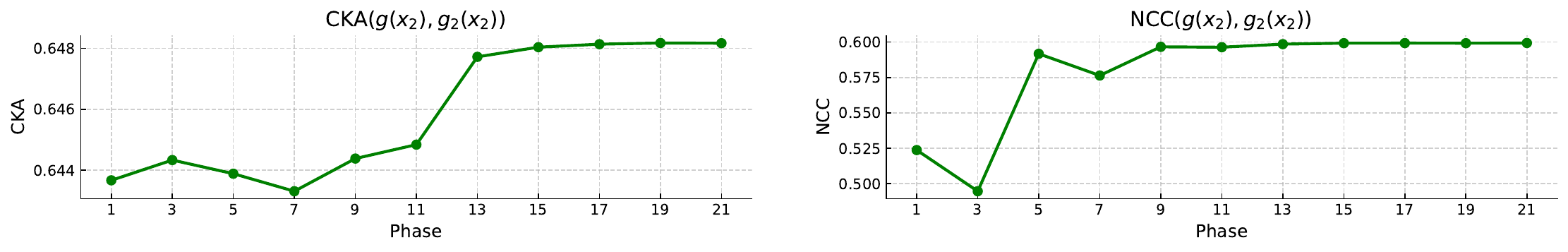}
    \end{subfigure}
    \caption{CKA (left column) and NCC (right column) between $\gbf_i$ trained by LDA on $D_i$ and $\gbf$ trained by U-LDA versus phase number $T$. Top and bottom rows show the results on brain ($\xbf_1$) and knee ($\xbf_2$) images, respectively.}
    \label{fig:corr_anals}
\end{figure}

\paragraph{Performance with various data sizes in target domains}

We also investigate how our proposed method performs with extremely limited data.
Specifically, we still consider the problems of cross-anatomy, cross-sampling ratio, and cross-modality as above, but test the performance of all methods compared with subsets of 5, 40, and 100 images randomly retrieved from $\hat{D}_1$ and $\hat{D}_2$. 
The results are presented in Figure \ref{fig:small-data}.
It is worth noting that our method (green bars) achieves the highest average reconstruction quality in PSNR in all tests. In particular, the PSNR obtained by U-LDA using only 5 images is higher than the PSNR achieved by the second-best method using 100 images.
\begin{figure}[t]
    \centering
    \begin{subfigure}[b]{0.85\textwidth}
    \includegraphics[width=\linewidth]{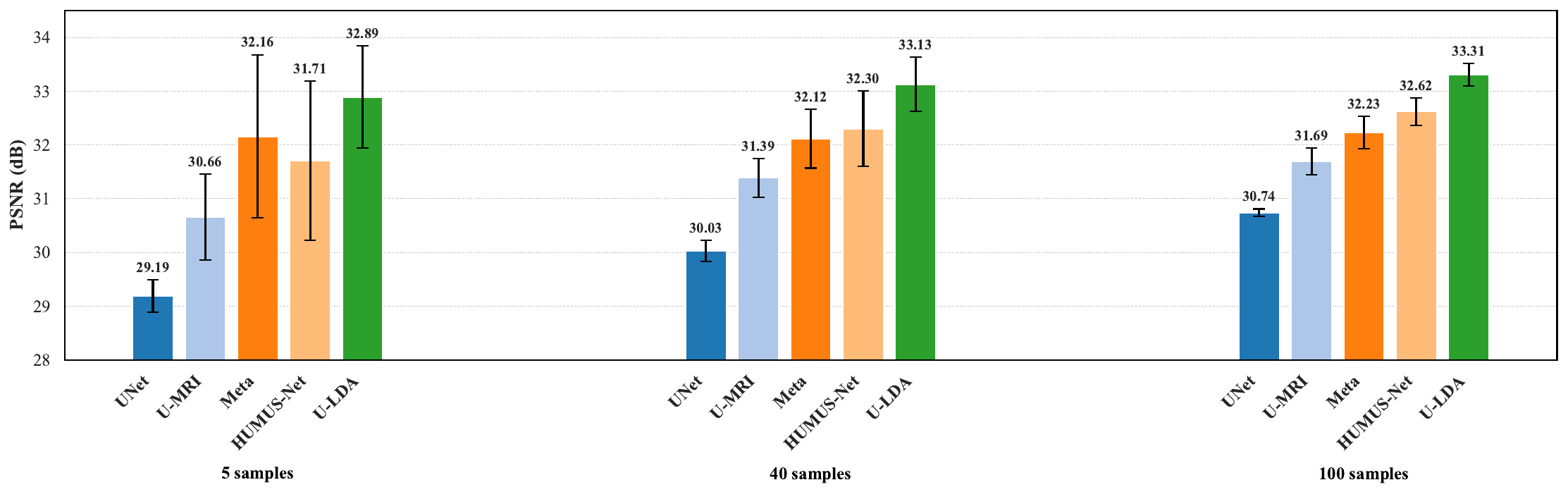}
    \caption{Cross-anatomy TL on target prostate images.}
    \end{subfigure}
    \begin{subfigure}[b]{0.85\textwidth}
        \includegraphics[width=\linewidth]{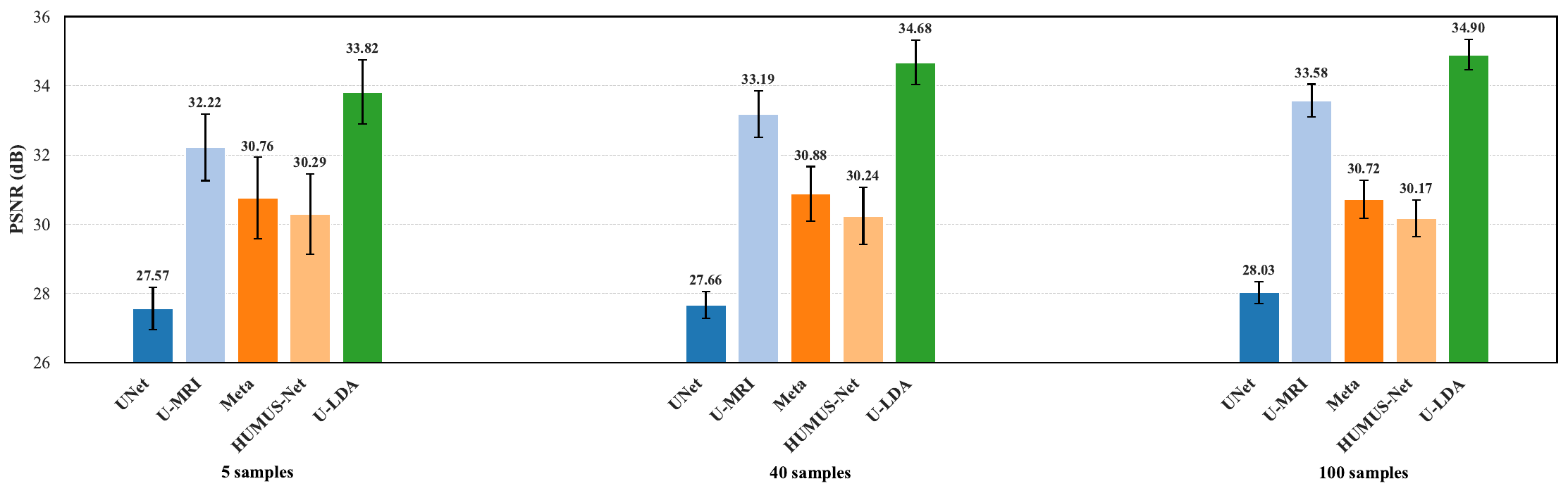}
        \caption{Cross-sampling-ratio on brain images of target 15\% sampling ratio.}
        \label{fig:smal-data-cross-sampling}
    \end{subfigure}
    \begin{subfigure}[b]{0.85\textwidth}
        \includegraphics[width=\linewidth]{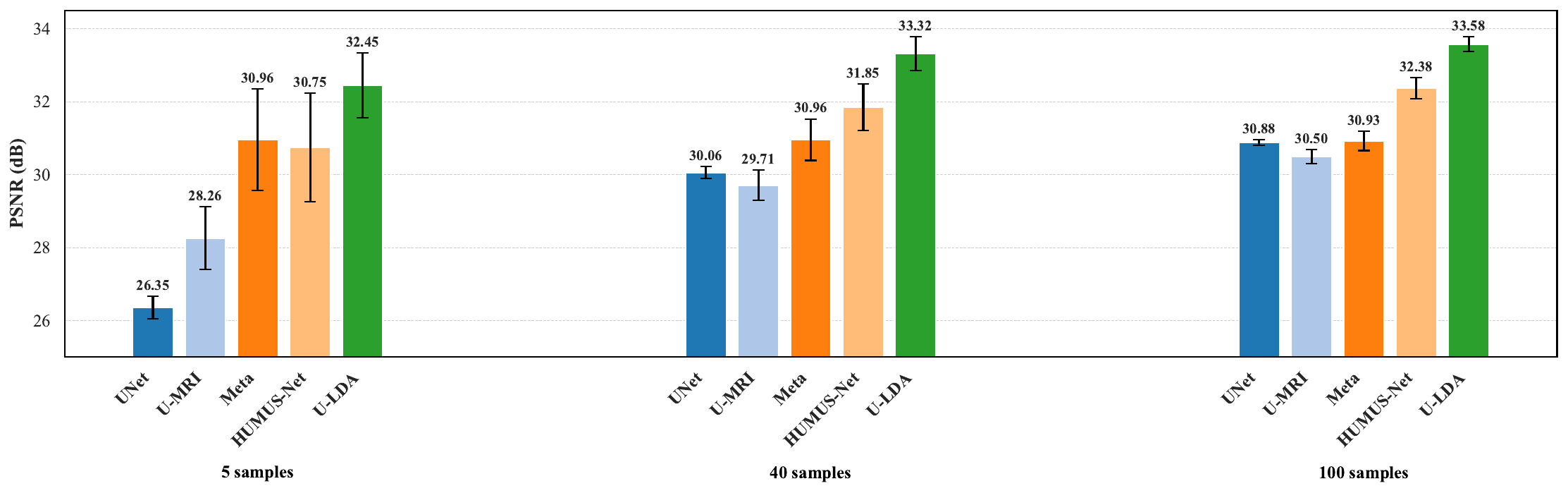}
        \caption{Cross-modality TL from natural images to target fastMRI images.}
        \label{fig:small-data-cross-modality}
    \end{subfigure}
    \caption{Average PSNR of reconstructed images with three different target domain training sample sizes (5, 40, and 100). Each bar represents the average PSNR and standard deviation (black interval) obtained by the corresponding method over five tests.}
    \label{fig:small-data}
\end{figure}
%

\paragraph{Technical component analysis}
We introduced three techniques to improve the reconstruction quality of our method, including using adapters (Adapter), initialization of $\gbf$ (Init), and data augmentation (AUG). 
We conduct several tests to check how they affect the empirical performance of U-LDA. The results are given in Table \ref{tab:ablation}. 
It appears that these techniques are indeed helpful to improve reconstruction quality in terms of higher PSNR values.
\begin{table}[h!]
\centering
\caption{PSNR of images reconstructed by U-LDA with ($\checkmark$) or without ($\times$) adapters (Adapter), initialization trick of $\gbf$ (Init), and data augmentation (AUG).}
\label{tab:ablation}
\centering
\begin{tabular}{cccccc}
\toprule
\multicolumn{3}{c}{Methods} & \multicolumn{3}{c}{Results (PSNR)} \\
\cmidrule(lr){1-3} \cmidrule(lr){4-6}
 {Adapter} & {Init} & {AUG} & {Cross-Anatomy} & {Cross-Sampling-Rate} & {Cross-Modality} \\
\midrule
$\times$ & $\times$ & $\times$ & 32.92 & 33.26 & 33.42 \\
$\checkmark$ & $\times$ & $\times$ & 32.84 & 33.48 &  33.45\\
$\checkmark$ & $\checkmark$ & $\times$ & 33.10 & 33.68 & 33.47 \\
$\checkmark$ & $\checkmark$ & $\checkmark$ & \textbf{33.31} & \textbf{34.90} & \textbf{33.58} \\
\bottomrule
\end{tabular}
\end{table}

\paragraph{Parameter and computation efficiency}
{We list the network sizes of the compared methods.
The standard UNet contains over $7\times10^6$ parameters.
DnCn has $2\times10^6$ to $5\times10^6$ parameters depending on the number of cascades.
HUMUS-Net has over $10^8$ parameters. 
By constrast, the universal feature extractor $\gbf$ in the proposed U-LDA has only 36{,}864 parameters, and each task-specific adapter $\hat{\hbf}_j$ has approximately 9,200 parameters.
The total number of parameters of U-LDA is typically smaller than $10^6$.}

{As all experiments are performed on the same hardware, we use wall clock hour to measure the computational costs of the compared methods. 
For each experiment, the dataset consists of 800--1200 images and is used by all compared methods.
The training time of UNet for 100 epochs requires two to three hours.
DnCn takes slightly more time due to its unrolled cascade structure. 
By contrast, HUMUS-Net is much more computationally expensive to train, and we limit its training to 50 epochs to maintain a comparable computational time. 
Meta requires 5 hours for 100 epochs of training.
The proposed U-LDA takes 0.5 hour for 100 epoch. 
It is worth pointing out that U-LDA only needs to train a small adapter for each transfer task and the training time is much smaller than the other methods in practice.
}

\paragraph{Hyperparameter analysis}
{There are two user-chosen hyperparameters in the proposed U-LDA method. The first one is the weight parameter $\alpha$ of the SSIM term in the training loss functions in \eqref{eq:solve-g-upper} and \eqref{eq:solve-h-upper}.
We repeat the cross-sampling experiment in Section \ref{subsubsec:cross-anatomy} with different values of $\alpha$, and present the average PSNR of images reconstructed by U-LDA in Table \ref{tab:alpha}.
}
\begin{table}[h!]
\centering
\caption{Average PSNR of images reconstructed by U-LDA using different values of SSIM weight parameter $\alpha$.}
\label{tab:alpha}
\begin{tabular}{cccc}
\toprule
$\alpha$ & 0.1 & 0.01 & 0.001\\
\midrule
PSNR & 36.45 & 37.88 & 37.69\\
\bottomrule
\end{tabular}
\end{table}

{The other hyperparameter is $\delta$ in the smoothed ReLU function. Larger $\delta$ yields more smoothed activation function (smaller Lipschitz constant of the gradient). We again repeat the cross-sampling experiment in Section \ref{subsubsec:cross-anatomy} with different values of $\delta$. The average PSNR of reconstructed images is shown in Table \ref{tab:delta}. 
}
\begin{table}[h!]
\centering
\caption{Average PSNR of images reconstructed by U-LDA using different values of smoothing level $\delta$ in the smoothed ReLU activation.}
\label{tab:delta}
\begin{tabular}{ccccc}
\toprule
$\delta$ & 0.0001 & 0.001 & 0.01 & 0.1 \\
\midrule
PSNR & 27.32 & 37.88 & 37.79 & 37.43  \\
\bottomrule
\end{tabular}
\end{table}

\paragraph{Current limitations}
{The proposed method is only tested on several public image datasets. It is unclear how it performs in real-world applications where additional restrictions, such as data availability and privacy, may present. From the algorithm perspective, we can only guarantee convergence to Clarke stationary points of the optimization problems. This does not ensure global optimality as the objective functions of consideration are non-convex. }

\section{Conclusions}
\label{sec:conclusion}

In this work, we proposed a unified transfer learning framework by integrating variational modeling, bi-level optimization, and unrolling network technique. Our method learns a powerful feature-extractor from large and heterogeneous data and adapters for problems with small data. The feature-extractor and adapters are learned through two specifically designed bi-level optimization problems, which are solved by a modified efficient learnable descent algorithm with provable convergence and iteration complexity. We tested our method, called U-LDA, to a variety of transfer learning problems related to MR image reconstruction. Numerical results demonstrate promising performance of U-LDA in solution quality, transfer ability, network efficiency in all the tests.

\section*{Acknowledgments}
This research is supported in part by National Science Foundation under grants DMS-2152960, DMS-2152961, DMS-2307466, DMS-2409868, and DMS-2510830.

\bibliographystyle{abbrv}
\bibliography{ref_tl}

\end{document}